\documentclass[runningheads]{llncs}
\usepackage[T1]{fontenc}
\usepackage{graphicx}
\usepackage{amsmath}
\usepackage{booktabs}
\usepackage{algorithm}
\usepackage[switch]{lineno}
\usepackage[noend]{algpseudocode}
\usepackage{tikz}
\usetikzlibrary{positioning, shapes.geometric, arrows.meta}
\usepackage{newcommand}
\usepackage{multirow}
\usepackage[hidelinks]{hyperref}
\usepackage{Explanation_font}

\usepackage{color}

\begin{document}
\title{Circuit Representations of Random Forests with Applications to XAI}
%
%
\author{Chunxi Ji\orcidID{0000-0002-4475-1987} \and
Adnan Darwiche\orcidID{0000-0003-3976-6735}}
\authorrunning{C. Ji et al.}
%
\institute{University of California, Los Angeles, Los Angeles CA 90095, USA 
\email{jich@cs.ucla.edu, darwiche@cs.ucla.edu}\\}
\maketitle              
\begin{abstract}
We make three contributions in this paper. First, we present an approach for compiling a random forest classifier into a set of circuits, where each circuit directly encodes the instances in some class of the classifier. We show empirically that our proposed approach is significantly more efficient than existing similar approaches. Next, we utilize this approach to further obtain circuits that are tractable for computing the complete and general reasons of a decision, which are instance abstractions that play a fundamental role in computing explanations. Finally, we propose algorithms for computing the robustness of a decision and all shortest ways to flip it. We illustrate the utility of our contributions by using them to enumerate all sufficient reasons, necessary reasons and contrastive explanations of decisions; to compute the robustness of decisions; and to identify all shortest ways to flip the decisions made by random forest classifiers learned from a wide range of datasets.

\keywords{Tractable circuits \and Robustness \and Contrastive explanations}

\end{abstract}

\section{Introduction}
Explaining the decisions made by machine-learning classifiers has been receiving increasing attention in AI. Of particular interest are tree-based classifiers such as random forests~\cite{10.1023/A:1010933404324} and gradient boosted trees~\cite{10.1214/aos/1013203451} as they tend to work well in practice, sometimes outperforming neural networks~\cite{10.5555/3600270.3600307,10.1016/j.inffus.2021.11.011,DBLP:journals/tnn/BorisovLSHPK24}. Popular \textit{model-agnostic} explanations such as LIME~\cite{DBLP:conf/kdd/Ribeiro0G16} and ANCHOR~\cite{DBLP:conf/aaai/Ribeiro0G18} can be used to explain the behavior of tree-based classifiers. However, these explanations are approximations and offer no formal guarantees of their sufficiency and necessity. 

To address this issue, many types of \textit{model-based} explanations which come with formal guarantees have been studied in the AI literature, with two popular ones being the \textit{sufficient reasons}~\cite{DBLP:conf/ecai/DarwicheH20} and the \textit{necessary reasons}~\cite{DBLP:conf/aaai/DarwicheJ22}. A sufficient reason (SR) for a decision was first introduced in~\cite{DBLP:conf/ijcai/ShihCD18} 
under the name \textit{PI-explanation}, later also
called an \textit{abductive explanation}~\cite{DBLP:conf/aaai/IgnatievNM19}. 
An SR  
is a minimal subset of the instance that is guaranteed to trigger the decision, and it answers the question ``why did the instance get assigned to class $c$ by the classifier?'' 
A necessary reason (NR) for a decision is a minimal property
of the instance that flips the decision if violated appropriately (by changing the instance). It answers the question ``what does it take to change the decision?''
and corresponds to a \textit{basic contrastive explanation} as formalized earlier in~\cite{DBLP:conf/aiia/IgnatievNA020}. A \textit{contrastive explanation (CE)}~\cite{lipton_1990,DBLP:journals/ai/Miller19} shows how a decision can be changed to a particular target decision and can be formalized using necessary reasons as we discuss later.

The \textit{complete reason} (CR) for a decision was introduced in~\cite{DBLP:conf/ecai/DarwicheH20} and can be viewed as an instance abstraction that is both sufficient and necessary for the corresponding decision~\cite{DBLP:conf/lics/Darwiche23}. The SRs and NRs for a decision are exactly the prime implicants and implicates of the complete reason~\cite{DBLP:conf/ecai/DarwicheH20,DBLP:conf/aaai/DarwicheJ22}. Hence, we can systematically compute all SRs and NRs for a decision by first obtaining the complete reason and then computing its prime implicants and implicates. The complete reason can be computed by applying an operator introduced in~\cite{DBLP:journals/jair/DarwicheM21} to the instance and a \textit{class formula} which encodes all instances in the same class. All contrastive explanations for a decision and a target class can also be computed as the prime implicates of a special complete reason obtained by merging multiple class formulas~\cite{DBLP:conf/aaai/DarwicheJ22,DBLP:conf/lics/Darwiche23}. When some classifier features are non-binary, improved explanations
can be obtained using the \textit{general reason} which is a better instance abstraction compared to the complete reason in this case~\cite{DBLP:conf/jelia/JiD23}.\footnote{The complete and general reasons coincide when all features are binary~\cite{DBLP:conf/jelia/JiD23}.}

In this paper, we first propose an algorithm that efficiently compiles a random forest classifier into a set of class formulas in the form of circuits. We then provide an algorithm to convert the circuits into tractable decision graphs that allow us to compute the complete and general reasons in time linear in the sizes of these decision graphs. We illustrate the utility of these algorithms by showing how they allow us to efficiently enumerate all SRs and NRs for decisions based on complete reasons. We also extend the approach in~\cite{DBLP:conf/aaai/DarwicheJ22,DBLP:conf/lics/Darwiche23} for computing CEs based on complete reasons to random forest classifiers, as the original approach assumes that every instance belongs to a unique class which does not generally hold for random forests. We further contribute two algorithms based on complete and general reasons, one for computing the robustness of a decision and another for computing all shortest ways to flip a decision. We finally report on the efficiency of our approaches using random forests learned from a wide range of datasets.

In terms of the number of trees that can be handled, 
the current state-of-the-art approaches for explaining random forest classifiers convert such classifiers into symbolic encodings and then query a SAT or MAX-SAT solver~\cite{DBLP:conf/ijcai/Izza021,DBLP:conf/aaai/IzzaIS024,DBLP:conf/ijcai/IzzaIR0S25}.
These approaches (normally geared towards computing one explanation) yield
encodings that contain many auxiliary variables so do not correspond to class formulas and cannot be directly used to compute complete and general reasons. Our approach for compiling random forest classifiers into tractable circuits is more costly, but it yields the complete and general reasons for the decision, which are powerful notions that allow us to efficiently enumerate all SRs, NRs, and CEs, and compute decision robustness as well as all shortest ways to flip the decision --- as illustrated by our empirical evaluation.
The approaches in~\cite{DBLP:journals/sttt/GossenS23,DBLP:conf/icaart/MurtoviSS25} compile random forest classifiers into decision graphs, from which class formulas can be easily extracted (in the form of circuits). But the extracted circuits are not tractable and cannot be used to compute the complete and general reason in linear time like we do. Moreover, our empirical evaluation will show that our compilation algorithm can generate class formulas of random forest classifiers much more efficiently than these approaches.\footnote{The primary purpose of~\cite{DBLP:journals/sttt/GossenS23,DBLP:conf/icaart/MurtoviSS25} is to compile random forest classifiers into more interpretable forms (i.e., decision graphs). But we use these approaches as a baseline for extracting class formulas as we are not aware of other approaches for this task.}

This paper is structured as follows. Section~\ref{sec:Preliminaries} provides technical preliminaries. Section~\ref{sec:sortnet} contains our circuit compilation algorithms. Section~\ref{sec:robustness} includes our algorithms for computing decision robustness and shortest decision flips. Section~\ref{sec:experiments} includes an empirical evaluation and Section~\ref{sec:conclusion} closes with concluding remarks.

\begin{figure}[tb]
        \centering
        \scalebox{0.8}{
        \begin{tikzpicture}[
        roundnode/.style={circle ,draw=black, thick},
        squarednode/.style={rectangle, draw=black, thick},
        ]
        \node[squarednode]     (X)                              {\Large $Y$};
        \node[roundnode]     (c1)       [below=of X, xshift = -0.8cm, yshift = 0.5cm] {\Large $c_1$};
        \node[squarednode]     (Y)       [below=of X, xshift = 0.8cm, yshift = 0.5cm] {\Large $X$};
        \node[squarednode]       (Z1)   [below=of Y, xshift = -0.8cm, yshift = 0.5cm] {\Large $Y$};
        \node[squarednode]       (Z2)   [below=of Y, xshift = 0.8cm, yshift = 0.5cm] {\Large $Z$};
        \node[roundnode]       (c2)   [below=of Z1, xshift = -0.8cm, yshift = 0.5cm] {\Large $c_1$};
        \node[roundnode]       (c3)   [below=of Z1, xshift = 0.8cm, yshift = 0.5cm] {\Large $c_2$};
        \node[roundnode]       (c4)   [below=of Z2, xshift = 0.8cm, yshift = 0.5cm] {\Large $c_3$};
        
        \draw[-latex, thick] (X.240) -- node [anchor = center, xshift = -4mm, yshift = 1mm] {$y_1$} (c1.north);
        \draw[-latex, thick] (X.300) -- node [anchor = center, xshift = 4mm, yshift = 1mm] {$y_{23}$} (Y.north);
        \draw[-latex, thick] (Y.240) -- node [anchor = center, xshift = -4mm, yshift = 1mm] {$x_1$} (Z1.north);
        \draw[-latex, thick] (Y.300) -- node [anchor = center, xshift = 4mm, yshift = 1mm] {$x_{23}$} (Z2.north);
        \draw[-latex, thick] (Z1.240) -- node [anchor = center, xshift = -4mm, yshift = 1mm] {$y_{3}$} (c2.north);
        \draw[-latex, thick] (Z1.300) -- node [anchor = center, xshift = -4mm] {$y_2$} (c3.north);
        \draw[-latex, thick] (Z2.240) -- node [anchor = center, xshift = 4mm] {$z_{12}$} (c3.north);
        \draw[-latex, thick] (Z2.300) -- node [anchor = center, xshift = 4mm, yshift = 1mm] {$z_3$} (c4.north);
        \end{tikzpicture}
        }
        \caption{A decision graph with ternary features ($X$, $Y$, $Z$) and classes ($c_1$, $c_2$, $c_3$). \label{fig:class-formula-dg}}
\end{figure}
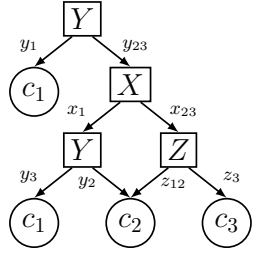

\section{Preliminaries}\label{sec:Preliminaries}
Although a classifier may take different forms, it is essentially a function that maps instances into classes. The class that an instance gets mapped into can be considered as the decision made by the classifier on that instance. The notion of discrete class formulas was proposed in~\cite{DBLP:conf/aaai/DarwicheJ22} to capture the behavior of a classifier, independently of its specific form, 
where each class is represented by a discrete formula that encodes all instances 
in that class.
We will represent a random forest classifier using discrete formulas, which apply
Boolean operators to discrete literals.
This generalizes an earlier proposal in \cite{DBLP:conf/ijcai/ShihCD18} which employed Boolean class formulas and was restricted to binary features and binary classes. We next review the syntax and semantics of discrete formulas.

\subsection*{Discrete Formulas: Syntax and Semantics}
We assume a finite set of variables \(\Sigma\) which represent 
classifier features. Each variable \(X \in \Sigma\) has a finite number of 
\textit{states} \(x_1, \ldots, x_n,\) \(n > 1\). 
For binary and categorical features, their states are provided by the datasets. For continuous features in decision trees/graphs, it is well known that they can be discretized into different intervals and therefore can be viewed as categorical features with a finite number of states; see, e.g.,~\cite{DBLP:conf/jelia/JiD23,DBLP:conf/lics/Darwiche23} for more details.
A \textit{literal} \(l\) for variable \(X\), called \(X\)-literal,
is a set of states such that \(\emptyset \subset l \subset \{x_1, \ldots, x_n\}\). 
We will often denote a literal such as \(\{x_1,x_3,x_4\}\) 
by \(x_{134}\) which reads: the state of variable \(X\) 
is either \(x_1\) or \(x_3\) or \(x_4\). A literal is \textit{simple} iff it contains
a single state. Hence, \(x_3\) is a simple literal but
\(x_{134}\) is not.
Since a simple literal corresponds to a state, 
these two notions are interchangeable. 
A \textit{discrete formula} is either true, false,
literal \(\lit\), negation \(\NOT \alpha\), conjunction
\(\alpha \AND \beta\) or disjunction \(\alpha \OR \beta\) 
where \(\alpha\), \(\beta\) are formulas. 
A term (clause) is a conjunction (disjunction) of literals. 
A negation normal form (\textit{NNF}) is a formula without negations. We say a term/clause/NNF is \textit{simple} iff 
it contains only simple literals.
An NNF is \textit{or-decomposable} if its disjuncts do not share variables. An NNF is \textit{monotone} if for each variable $X$, the $X$-literals are simple and equal. 
A \textit{world} maps each variable in \(\Sigma\) to
one of its states and is typically denoted by \(\w\).
A world \(\w\) is called a \textit{model} of formula \(\alpha\),
written \(\w \models \alpha\), iff \(\alpha\) is satisfied 
by \(\w\) (that is, \(\alpha\) is true at \(\w\)).
Formula \(\alpha\) implies formula \(\beta\), 
written \(\alpha \models \beta\), iff
every model of \(\alpha\) is also a model of \(\beta\). 
A term $\tau_1$ subsumes another term $\tau_2$ iff $\tau_2 \models \tau_1$. A clause $\sigma_1$ subsumes another clause $\sigma_2$ iff $\sigma_1 \models \sigma_2$.
An {\em implicant} of a discrete formula \(\Delta\) is a term \(\delta\) such that \(\delta \models \Delta\).
The implicant is {\em prime} iff no other implicant \(\delta^*\) is such that \(\delta \models \delta^*\).
An {\em implicate} of $\Delta$ is a clause \(\delta\) such that \(\Delta \models \delta\).
The implicate is {\em prime} iff no other implicate \(\delta^*\) is such that \(\delta^* \models \delta\). Finally, $\SetV(\Delta)$ denotes the set of all variables in $\Delta$.

\subsection*{Decision Graphs}

A decision graph generalizes a decision tree by allowing a node to have multiple parents. Each internal node of a decision graph tests a variable (i.e., feature) and each leaf node has a \textit{label;} see Figure~\ref{fig:class-formula-dg}.
A decision graph with labels $c_1, \ldots, c_k$ can be viewed as a classifier, which we can be represented by mutually exclusive class formulas \(C^1, \ldots, C^k\) that encode the instances in classes $c_1, \ldots, c_k.$
An \textit{instance} is a conjunction of literals $x_i,$ one literal for each feature $X.$
An instance \(\instance\) is in class \(c_i\) iff \(\instance \models C^i\).
The class formulas for the classifier in Figure~\ref{fig:class-formula-dg} are 
$C^1 = y_{1}\OR (x_1 \AND y_3)$ (encodes \(12\) instances), 
$C^2 = (x_1 \AND y_2) \OR (x_{23} \AND y_{23} \AND z_{12})$ (encodes \(11\) instances) and 
$C^3 = x_{23} \AND y_{23} \AND z_{3}$ (encodes \(4\) instances). 
Here, instance $\instance = x_2 \AND y_2 \AND z_3$ is in
class \(c_3\) since $\instance \models C^3$.

\subsection*{Random Forests}

A random forest \cite{10.1023/A:1010933404324} is a classifier consisting of a collection of decision trees. A random forest classifier makes a decision on an instance by majority voting. Each decision tree is evaluated at the instance and votes for a class, and the class that receives the most votes is picked. In case of ties, every class that receives the most votes could be picked so there might be instances that belong to multiple classes. For a random forest classifier $R$ with trees $T_1, \ldots, T_n$ and classes $c_1, \ldots, c_k$, we will use $C^i_j$ to denote the class formula of decision tree $T_j$ for class $c_i$. We will use $R^i$ to denote the class formula of the random forest for class $c_i$. 
Suppose we have five trees and three classes ($n=5$ and $k=3$) and let $\instance$ be an instance such that $\instance \models C_1^1 \wedge C_2^1 \wedge C_3^2 \wedge C_4^2 \wedge C_5^3$. That is, when evaluated at $\instance$, $T_1$ and $T_2$ vote for $c_1$, $T_3$ and $T_4$ vote for $c_2$, and $T_5$ votes for $c_3$. Since both $c_1$ and $c_2$ receive most votes, we have $\instance \models R^1$ and $\instance \models R^2$. Hence, it is possible for an instance to satisfy multiple class formulas and therefore belong to multiple classes. 

\subsection*{The Reasons Behind Decisions}

\begin{figure}[tb]
        \centering
        \scalebox{0.8}{
        \begin{tikzpicture}[
        roundnode/.style={text width = 0.55cm, circle ,draw=black, thick, text badly centered},
        squarednode/.style={rectangle, draw=black, thick, text badly centered},
        ]
        \node[squarednode]      (Age)                              {\large \fAge};
        \node[squarednode]        (Weight)       [below=of Age, xshift = -0.9cm, yshift = 0.4cm] {\large \fWeight};
        \node[roundnode]        (No1)       [below=of Age, xshift = 0.9cm, yshift = 0.5cm] {\fno};
        \node[roundnode]        (Yes1)       [below=of Weight, xshift = -2.3cm, yshift = 0.2cm] {\fyes};
        \node[squarednode]        (BloodType1)       [below=of Weight, yshift = 0cm] {\large \fBType};
        \node[squarednode]        (BloodType2)       [below=of Weight, xshift = 3cm, yshift = 0cm] {\large \fBType};
        \node[roundnode] (Yes2) [below=of BloodType1, xshift = -1.3cm, yshift = 0.3cm] {\fyes};
        \node[roundnode] (No2) [below=of BloodType1, xshift = 0cm, yshift = 0.3cm] {\fno};
        \node[roundnode] (Yes3) [below=of BloodType2, xshift = -1cm, yshift = 0.3cm] {\fyes};
        \node[roundnode] (No3) [below=of BloodType2, xshift = 0.5cm, yshift = 0.3cm] {\fno};
        
        \draw[-latex, thick] (Age.240) -- node [anchor = center, xshift = -5mm, yshift = 1mm] {\large $\geq 55$} (Weight.north);
        \draw[-latex, thick] (Age.300) -- node [anchor = center, xshift = 5mm, yshift = 1mm] {\large $< 55$} (No1.north);
        \draw[-latex, thick] (Weight.240) --  node [anchor = center, xshift = -12mm] {\large \fOWeight} (Yes1.north);
        \draw[-latex, thick] (Weight.270) --  node [anchor = center, xshift = -0mm] {\large \fUWeight} (BloodType1.north);
        \draw[-latex, thick] (Weight.300) --  node [anchor = center, xshift = 10mm] {\large \fNom} (BloodType2.north);
        \draw[-latex, thick] (BloodType1.240) --  node [anchor = center, xshift = -9mm] {\large \ftA, \ftB, \ftAB} (Yes2.north);
        \draw[-latex, thick] (BloodType1.270) --  node [anchor = center, xshift = 3mm] {\large \ftO} (No2.north);
        \draw[-latex, thick] (BloodType2.240) --  node [anchor = center, xshift = -6mm] {\large \ftA, \ftB} (Yes3.north);
        \draw[-latex, thick] (BloodType2.270) --  node [anchor = center, xshift = 6mm] {\large \ftAB, \ftO} (No3.north);
        \end{tikzpicture}
        }
        \caption{A classifier of patients susceptible to a certain disease from~\cite{DBLP:conf/jelia/JiD23}. \label{fig:exp-example}}
\end{figure}

We next review the explanations we will cover in this paper by using an example from~\cite{DBLP:conf/jelia/JiD23} relating to a decision-tree classifier with three features ($\Age$, $\BType$, $\Weight$) which is depicted in Figures~\ref{fig:exp-example}. Consider a 
patient Susan with $\Age \GE 55$, $\eql{\BType}{\tA}$ and
$\eql{\Weight}{\OWeight}$, leading to a $\yes$ decision by the classifier.
The sufficient reasons (SRs) for the decision on Susan are
$(\Age \GE 55 \AND \eql{\BType}{\tA})$ and $(\Age \GE 55 \AND \eql{\Weight}{\OWeight})$. SRs correspond to minimal subsets of the instance which are guaranteed to trigger the decision. For example, the SR $(\Age \GE 55 \AND \eql{\BType}{\tA})$ tells us that Susan's $\Age$ and $\BType$ are sufficient to trigger the $\yes$ decision (but neither feature alone can do this). The second SR, $(\Age \GE 55 \AND \eql{\Weight}{\OWeight}),$ tells us that Susan's $\Age$ and $\Weight$ are also sufficient to trigger
the $\yes$ decision.
The general sufficient reasons (GSRs) introduced in~\cite{DBLP:conf/jelia/JiD23} are improved versions of SRs as they contain more information. 
The GSRs for the decision on Susan are
$(\Age\GE 55 \AND \BType \IN \{\tA, \tB\})$ and $(\Age \GE 55 \AND \eql{\Weight}{\OWeight})$. The first GSR tells not only that Susan's $\Age$ and $\BType$ are sufficient to trigger the decision (like the first SR above) but also that the
decision would still be triggered even if Susan's $\BType$ was $\tB.$
GSRs subsume (encode) all SRs but contain more information; see~\cite{DBLP:conf/jelia/JiD23}. 
Syntactically, a GSR corresponds to a conjunction of literals unlike a SR which corresponds to a conjunction of \textit{simple} literals.

The necessary reasons (NRs) for the $\yes$ decision on Susan 
are $(\Age \GE 55)$ and $(\eql{\Weight}{\OWeight} \OR \eql{\BType}{\tA})$. NRs identify minimal subsets of features that will flip the decision if changed appropriately --- but 
they do not tell us how to change these features. As formulated in~\cite{DBLP:conf/aaai/DarwicheJ22}, an NR is 
a minimal property (clause) of the instance which would flip the decision if violated appropriately (by changing the instance). 
We can flip the decision on Susan by violating the NR $(\eql{\Weight}{\OWeight} \OR \eql{\BType}{\tA})$ which can be done by changing
the values of $\Weight$ and $\BType$ in the instance. There are six possible changes to $\Weight$ and $\BType$ that will violate 
this NR, some will flip the decision but others will not. For example, the changes  
$\eql{\Weight}{\Nom}, \eql{\BType}{\tO}$ and $\eql{\Weight}{\UWeight}, \eql{\BType}{\tAB}$
will both violate the NR but only the first one will flip the decision. This issue was addressed in~\cite{DBLP:conf/jelia/JiD23} by introducing general necessary reasons (GNRs), which subsume NRs and come with stronger guarantees.
The GNRs for the decision on Susan are 
$(\Age\GE 55)$, $(\BType \IN \{\tA, \tB, \tAB\} \OR$  $\eql{\Weight}{\OWeight}\})$ and $(\BType \IN \{\tA, \tB\} \OR \Weight\IN\{\UWeight, \OWeight\})$.
If the instance is changed to violate any GNR, the decision will change regardless of how the GSR was violated. 
For example, if we set $\BType$ to $\tAB$ and $\Weight$ to $\Nom$, the third GNR will be violated and the decision on Susan becomes $\no$. 
GNRs can contain arbitrary literals, unlike NRs which contain only simple literals.

There are different methods for computing SRs, NRs and their generalizations (GSRs, GNRs). One method, which we adopt in this paper, is based on the notion of an \textit{instance abstraction.}
Intuitively, an instance abstraction is a logical condition on the instance that satisfies some syntactic restriction yet is both sufficient and necessary to trigger the decision on that instance; see~\cite{DBLP:conf/lics/Darwiche23}.
The \textit{complete reason}~\cite{DBLP:conf/ecai/DarwicheH20} is one such abstraction with the syntactic restriction that it is an NNF whose literals are \textit{contained} in the instance (hence, must be simple literals). 
The \textit{general reason}~\cite{DBLP:conf/jelia/JiD23}  is another abstraction with the syntactic restriction that it is an NNF whose literals are \textit{implied} by the instance (hence, can be arbitrary literals).\footnote{The general reason is weaker than the complete reason, and the two notions coincide with one another when all features are binary~\cite{DBLP:conf/jelia/JiD23}.}
SRs are the prime implicants of the complete reason~\cite{DBLP:conf/ecai/DarwicheH20}, and NRs are its prime implicates~\cite{DBLP:conf/aaai/DarwicheJ22}. Moreover, GSRs and GNRs are respectively the variable-minimal prime implicants and prime implicates of the general reason; see~\cite{DBLP:conf/jelia/JiD23} for the definition of variable-minimality and algorithms that computes GSRs and GNRs from the general reason.
The complete reason for the above decision on Susan is (equivalent to) $(\Age \GE 55 \AND (\eql{\BType}{\tA} \OR \eql{\Weight}{\OWeight}))$,
while the general reason is (equivalent to) $(\Age \GE 55 \AND (\BType \in \{\tA, \tB, \tAB\} \OR \Weight \in \{\OWeight, \Nom\}) \AND (\BType \in \{\tA, \tB\} \OR \Weight \in \{\OWeight, \UWeight\}))$.
The complete and general reasons can both be computed efficiently (in linear time) if class formulas have specific properties (e.g., or-decomposable); see~\cite{DBLP:conf/lics/Darwiche23} for a review of corresponding results. This is why our approach rests on obtaining class formulas in specific forms, to facilitate the computation of complete and general reasons, which will further facilitate the computation of SRs, NRs and their generalizations. 

\section{Compiling Random Forests into Circuits}\label{sec:sortnet}

\begin{figure}[tb]
\centering
\begin{tikzpicture}[xscale=0.7, yscale=0.7, thick]
  \def\n{8}
  \def\xstep{1.0}
  \def\ystep{0.8}
  \def\dx{0.25} 

  \foreach \i in {1,...,\n} {
    \draw (0, -\i*\ystep) -- (7*\xstep, -\i*\ystep);
    \node[left] at (0, -\i*\ystep) {$x_{\i}$};
    \node[right] at (7*\xstep, -\i*\ystep) {$y_{\i}$};
  }

  \tikzset{dot/.style={circle,fill,inner sep=1.2pt}}

  \newcommand{\comp}[3]{
    \draw (#3, -#1*\ystep) -- (#3, -#2*\ystep);
    \node[dot] at (#3, -#1*\ystep) {};
    \node[dot] at (#3, -#2*\ystep) {};
  }

  \comp{1}{2}{1.0*\xstep}
  \comp{3}{4}{1.0*\xstep}
  \comp{5}{6}{1.0*\xstep}
  \comp{7}{8}{1.0*\xstep}

  \comp{1}{3}{2.0*\xstep}
  \comp{2}{4}{2.0*\xstep+\dx}
  \comp{5}{7}{2.0*\xstep}
  \comp{6}{8}{2.0*\xstep+\dx}

  \comp{2}{3}{3.0*\xstep}
  \comp{6}{7}{3.0*\xstep}

  \comp{1}{5}{4.0*\xstep-\dx}
  \comp{2}{6}{4.0*\xstep}
  \comp{3}{7}{4.0*\xstep+\dx}
  \comp{4}{8}{4.0*\xstep+2*\dx}

  \comp{3}{5}{5.0*\xstep}
  \comp{4}{6}{5.0*\xstep+\dx}

  \comp{2}{3}{6.0*\xstep}
  \comp{4}{5}{6.0*\xstep}
  \comp{6}{7}{6.0*\xstep}

\end{tikzpicture}
\caption{An odd-even network with eight inputs/outputs. \label{fig:odd-even-sortnet}}
\end{figure}

\begin{figure}[tb]
    \centering
    \includegraphics[width=0.6\linewidth]{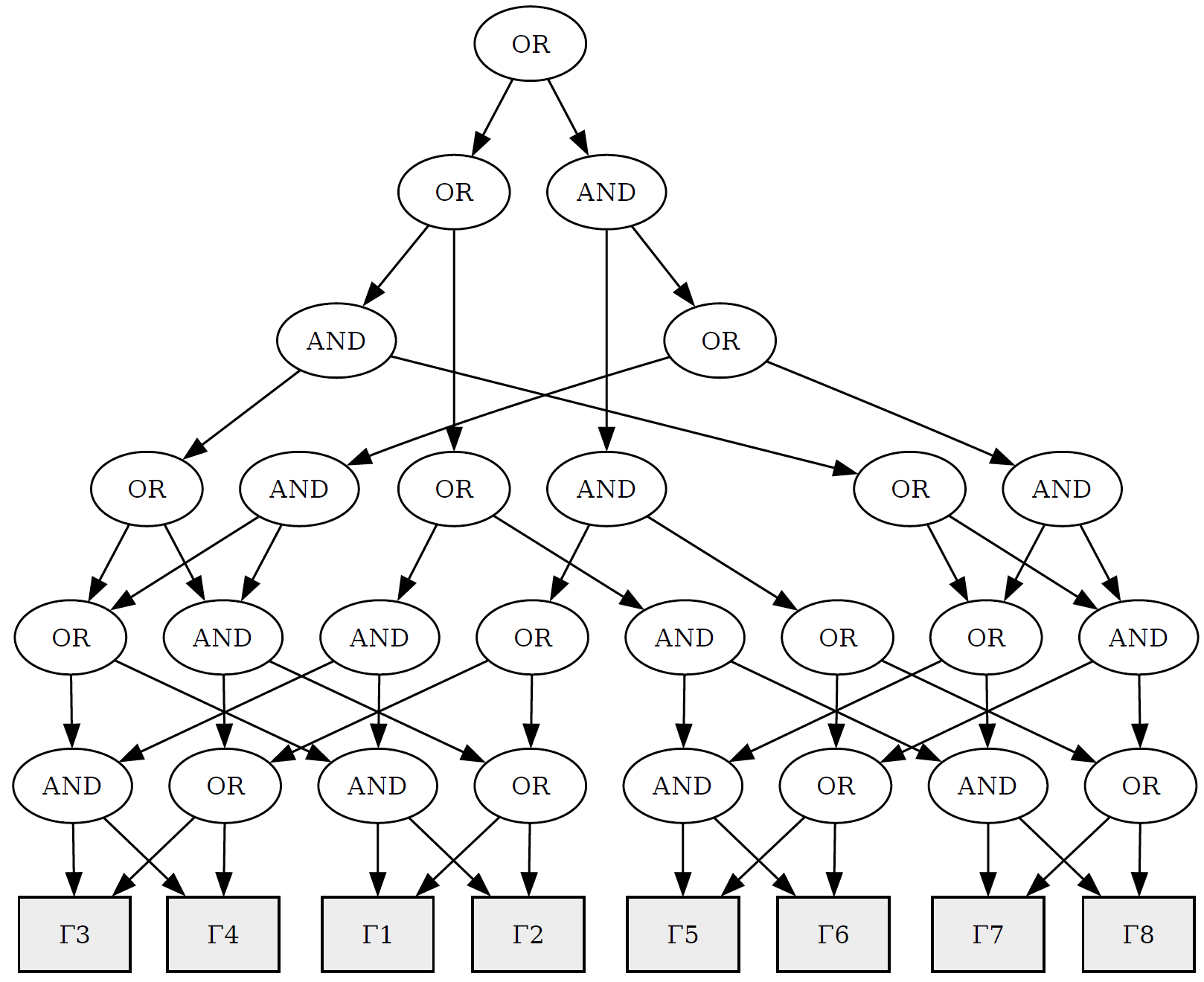}
    \caption{The NNF circuit corresponding to output $\Delta_4$ of the sorting network in Figure~\ref{fig:odd-even-sortnet}.}
    \label{fig:y4}
\end{figure}

A common approach to symbolically encode a random forest classifier is to use auxiliary Boolean variables when representing majority constraints; see, e.g.,~\cite{DBLP:conf/ijcai/Izza021,DBLP:conf/ijcai/IzzaIR0S25}. Consider a random forest with decision trees $T_1, T_2, T_3$ and classes $c_1, c_2$. 
The following encoding with auxiliary variables $a_1, a_2, a_3$ captures the instances in class $c_1$:
\(
a_1 \Leftrightarrow C_1^1, a_2 \Leftrightarrow C_2^1, a_3 \Leftrightarrow C_3^1, a_1 + a_2 + a_3 \ge 2.\)
These encodings are usually converted into CNFs which are used to compute explanations by applying carefully crafted queries using SAT, SMT, or MAX-SAT solvers.\footnote{Encoding cardinality constraints, like $a_1 + a_2 + a_3 \ge 2,$ as CNFs also requires additional auxiliary variables; see, e.g. \cite{DBLP:conf/cp/BailleuxB03,10.1007/11564751_73,DBLP:conf/sat/AsinNOR09}.} Auxiliary variables are necessary for these encodings as it is well known that CNFs with no auxiliary variables require exponential size to encode majority functions; see, e.g.,~\cite{10.5555/2190632,DBLP:journals/mst/EmdinKMS24}. Due to the presence of auxiliary variables, these encodings do not directly correspond to class formulas; for example, they can evaluate to false under an instance that satisfies the class formula if auxiliary variables are set carefully. Therefore, we next 
present a method for efficiently compiling a random forest classifier into class formulas in the form of NNF circuits (i.e., circuits with only literals, conjunctions and disjunctions).

The method is based on sorting networks, a tool from theoretical computer science which can be used to efficiently generate majority NNF circuits~\cite{10.1145/1468075.1468121,DBLP:journals/ppl/Parberry92}. We next review these networks, particularly the ones known as odd-even sorting networks~\cite{10.1145/1468075.1468121} as we found them favorable to pairwise sorting networks~\cite{DBLP:journals/ppl/Parberry92} in terms of compilation time.
There are other methods for generating majority circuits, beyond sorting networks, but they are not as practical even though some achieve better theoretical results in terms of circuit depths; see~\cite{DBLP:conf/approx/Dobrokhotova-Maikova24}.

A sorting network receives $n$ inputs $x_1, \ldots, x_n$ and permutes them to generate the sorted outputs $y_1, \ldots, y_n$.\footnote{As is customary, we will assume that $n$ is a power of $2$ without loss of generality (as we can always pad inputs).} The sorting network consists of comparators, where each comparator takes two inputs $a_1, a_2$ and generates the sorted output $b_1, b_2.$ In particular, $b_1$ is the maximal value of the inputs, $\max(a_1,a_2)$, and $b_2$ is the minimal value, $\min(a_1,a_2)$. Figure~\ref{fig:odd-even-sortnet} shows an example of an odd-even sorting network introduced in~\cite{10.1145/1468075.1468121} for eight inputs.

\begin{algorithm}[tb]
\begin{footnotesize}
\caption{Construct sorting network with $n$ inputs
\label{alg:Sortnet}}
\begin{algorithmic}[1]
\Require circuits $x[1, \ldots, n]$, $n$ is power of $2$
\Ensure circuits $y[1, \ldots, n]$
\Function{Sortnet}{$x, n$}
\If{$n == 1$}
\Return $x$
\EndIf
\State $m \gets n/2$
\State $A \gets \Call{Sortnet}{x[1, \ldots, m], m}$
\State $B \gets \Call{Sortnet}{x[m + 1, \ldots, n], m}$
\State \Return $\Call{Merge}{A, B, m}$
\EndFunction
\end{algorithmic}
\end{footnotesize}
\end{algorithm}

\begin{algorithm}[tb]
\begin{footnotesize}
\caption{Merge two sorted lists
\label{alg:Merge}}
\begin{algorithmic}[1]
\Require sorted lists $A[1, \ldots, m]$ and $B[1, \ldots, m]$, $m$ is power of $2$
\Ensure sorted list $y[1, \ldots, 2m]$
\Function{Merge}{$A, B, m$}
\If{$m == 1$}
\State \Return $[A[1] \vee B[1], A[1] \wedge B[1]]$
\EndIf
\State $A_1 \gets A[1,3,5,\ldots,m-1]$, $A_2 \gets A[2,4,6,\ldots,m]$
\State $B_1 \gets B[1,3,5,\ldots,m-1]$, $B_2 \gets B[2,4,6,\ldots,m]$
\State $C \gets \Call{Merge}{A_1, B_1, m/2}$
\State $D \gets \Call{Merge}{A_2, B_2, m/2}$
\State $y \gets [C[1]]$
\For{$i = 2 \text{ to } m$}
\State Append $C[i] \vee D[i - 1]$ to $y$
\State Append $C[i] \wedge D[i - 1]$ to $y$
\EndFor
\State Append $D[m]$ to $y$
\State \Return $y$
\EndFunction
\end{algorithmic}
\end{footnotesize}
\end{algorithm}
\begin{algorithm}[t]
\begin{footnotesize}
\caption{Construct class formula
\label{alg:RF-classformula}}
\begin{algorithmic}[1]
\Require decision trees $T[1, \ldots, n]$, classes $c[1, \ldots, k]$
\Ensure class formula of class $c[i]$ of the random forest
\Function{RFClassFormula}{$T, n, c, k, i$}
\State $m \gets$ the smallest power of $2$ no less than $n$
\If{$k == 2$}
\For{$l = 1 \text{ to } n$}
\State $x[l] \gets$ class formula for $c[i]$ of $T[l] $
\EndFor
\For{$l = n + 1 \text{ to } m$}
\State $x[l] \gets$ false
\EndFor
\State \Return $\Call{Sortnet}{x, m}[\lceil n/2 \rceil]$
\label{ln:sortnet-out1}
\EndIf
\State $R^i \gets$ true
\For{$j = 1 \text{ to } k$, $j \not = i$}
\For{$l = 1 \text{ to } n$}
\State $x[2l] \gets$ neg. of class formula for $c[j]$ of $T[l]$
\State $x[2l + 1] \gets$ class formula for $c[i]$ of $T[l]$
\EndFor
\For{$l = n + 1 \text{ to } m$}
\State $x[2l] \gets$ false;\ \ \ $x[2l + 1] \gets$ false \label{line:conjoin-dds}
\EndFor
\State $R^i \gets R^i \wedge \Call{Sortnet}{x, 2m}[n]$
\label{ln:sortnet-out2}
\label{ln:conjoin-F}
\EndFor
\State \Return $R^i$
\EndFunction
\end{algorithmic}
\end{footnotesize}
\end{algorithm}

Suppose now that $\Gamma_1,\ldots,\Gamma_n$ are NNF circuits, each over features $F_1,\ldots,F_m$. We will pass these NNF circuits as inputs to the sorting network and interpret a comparator as disjoining/conjoining its inputs to compute the max/min values of these inputs. That is, when a comparator is passed NNF circuits $\alpha,\beta$ as inputs, it will output the NNF circuits $\alpha\vee\beta, \alpha\wedge\beta.$ Hence, the final outputs of the sorting network will be NNF circuits $\Delta_1,\ldots,\Delta_n$ over features $F_1,\ldots,F_m$.
Moreover, the NNF circuit $\Delta_{n/2}$ will evaluate to true iff at least $n/2$ of the input circuits $\Gamma_1,\ldots,\Gamma_n$ will evaluate to true (at some setting of the features $F_1,\ldots,F_m$). Figure~\ref{fig:y4} depicts the NNF circuit $\Delta_4$ (fourth output) of the sorting network in Figure~\ref{fig:odd-even-sortnet}.
Algorithm~\ref{alg:Sortnet} provides the pseudocode for constructing an odd-even sorting network~\cite{10.1145/1468075.1468121} and uses 
Algorithm~\ref{alg:Merge} as a helper. This pseudocode is based on the analysis of odd-even sorting networks in~\cite{DBLP:conf/lpar/CodishZ10}.
Algorithm~\ref{alg:Sortnet} requires the number of input circuits to be a power of $2$, which can be ensured by padding the input with trivial circuits corresponding to false. 
The algorithm returns an ordered (sorted) list of NNF circuits.\footnote{One can also use sorting networks to efficiently encode majority functions or cardinality constraints as CNFs but with auxiliary variables; see, e.g.,~\cite{10.5555/646542.696213,DBLP:journals/jsat/EenS06}.}

We next show how to use the construction embodied in Algorithm~\ref{alg:Sortnet}
to obtain
class formulas for random forests. We treat binary classes first and then generalize to multi classes.

\subsection*{Class Formulas for Binary Classes}

Consider a random forest classifier $R$ with decision trees $T_1, \ldots, T_8$ and binary classes $c_1,c_2$. 
To obtain a class formula for $c_1,$ we first construct class formulas $C_1^1, \ldots, C_8^1$ for decision trees 
$T_1, \ldots, T_8$ in the form of NNF circuits, which can be done efficiently as shown in~\cite{DBLP:conf/aaai/DarwicheJ22}. We then pass the circuits 
$C_1^1, \ldots, C_8^1$ to a sorting network with eight inputs, as discussed earlier, and collect the fourth circuit output which will be equivalent to class formula $R^1.$ That is, $R^1$ will evaluate to true at an instance when at least four of the decision trees will assign class $c_1$ to that instance. Class formula $R^2$ can be obtained similarly.
More generally, consider a random forest classifier $R$ with decision trees $T_1, \ldots, T_n$ and binary classes $c_1, c_2$ where $n$ is a power of $2.$ 
Using Algorithm~\ref{alg:Sortnet}, we obtain the class formula $R^i$ by calling $\text{Sortnet}([C^i_{1},\ldots,C^i_{n}], n)[n/2]$ for $i = 1,2$
where $C^i_l$ is the class formula for decision tree $T_l$ and class $c_i.$ If $n$ is not a power of $2$, we pad the inputs by adding false circuits. Let $m$ be the smallest power of $2$ that is larger than $n$ and $x_1, \ldots, x_m$ be a list where $x_l = C^i_l$ for $l \le n$ and $x_{l} = \text{false}$ for $l > n$. Then the class formula $R^i = \text{ Sortnet}([x_{1},\ldots,x_{m}], m)[\lceil n/2 \rceil]$ for $i = 1,2$.

\subsection*{Class Formulas for Multi Classes}

The technique we just discussed does not work for more than two classes as we show next. Consider a random forest with $10$ decision trees $T_1, \ldots, T_{10}$ and three classes $c_1, c_2, c_3$. Suppose that for instance $I_1$, $4$ trees vote for class $c_1$, $3$ trees vote for class $c_2$, and $3$ trees vote for class $c_3$. Suppose further that for instance $I_2$,  $4$ trees vote for class $c_1$, one tree votes for class $c_2$, and $5$ trees vote for class $c_3$. For both instances, the number of class formulas $C^1_l$ that evaluate to true is the same. However, the random forest classifies $I_1$ into class $c_1$ but classifies $I_2$ into class $c_3$. Therefore, to construct class formula $R^1$ of the random forest, the inputs to the sorting network cannot simply be the class formulas $C^1_{1},\ldots,C^1_{10}$, as they do not contain enough information. 

We next propose a more refined technique that can handle multiple classes. 
Let $T_1, \ldots, T_n$ be the decision trees in a random forest classifier $R$ and $c_1, \ldots, c_k$ be the classes. We will construct the class formula $R^i$ in two steps. First, for each $j \not = i,$ 
we construct an NNF circuit $F^{ij}(T_1, \ldots, T_n)$ that encodes instances for which ``class $c_i$ has no fewer votes than class $c_j.$'' 
Next, we conjoin $F^{ij}(T_1, \ldots, T_n)$ for all $j \not = i$ to obtain an NNF circuit representing class formula $R^i.$

Let $\COUNT_\instance(\alpha_1,\ldots,\alpha_n)$ denote the number of formulas among $\alpha_1, \ldots, \alpha_n$ that are satisfied by instance $\instance$, i.e., the cardinality of the set $\{i \mid \instance \models \alpha_i\}$.
We then have
$\COUNT_\instance(C^i_{1},\ldots, C^i_{n}) \geq \COUNT_\instance(C^j_{1},\ldots,C^j_{n})$
iff 
$\COUNT_\instance(C^i_{1},\ldots,C^i_{n}) + 
\COUNT_\instance(\neg C^j_{1},\ldots,\neg C^j_{n}) \geq n.$
Therefore, for a sorting network with inputs $C^i_{1},\ldots,C^i_n, \neg C_{1}^j, \ldots, \neg C_{n}^j$ and outputs $y_{1},\ldots,y_{2n}$, the output $y_n$ is satisfied by instance $\instance$ iff class $c_i$ receives no fewer votes than class $c_j$ for that instance. Hence, output $y_n$ is exactly the NNF circuit $F^{ij}(T_1, \ldots, T_n)$, which can be generated using Algorithm~\ref{alg:Sortnet}.\footnote{Formulas $C^i_{l}$ and $\neg C_{l}^j$ can be obtained as NNF circuits from tree $T_l$ in linear time using a closed-form in~\cite{DBLP:conf/aaai/DarwicheJ22}.
It is also worth mentioning that we can omit the first layer of $F^{ij}(T_1, \ldots, T_n)$ if we order the sorting network inputs as follows: $\neg C_{1}^j, C^i_{1},\neg C_{2}^j,C^i_{2}, \ldots,\neg C_{n}^j, C^i_n$. This is because if a tree votes for class $c_i$, it must not vote for class $c_j$. Thus, we can remove the comparators for $\neg C_{l}^j$ and $C_{l}^i$ for all $l$ from the sorting network as $C_l^i \models \neg C_{l}^j$: the inputs are in the form of sorted pairs. }
We can finally construct the formula 
$R^i$ for class $c_i$ by conjoining the circuits $F^{ij}(T_1, \ldots, T_n)$ for all $j \not = i$:
$R^i = \bigwedge_{j \not = i}F^{ij}(T_1, \ldots, T_n)$.

Algorithm~\ref{alg:RF-classformula} puts everything we discussed together for computing the class formulas of a random forest. It works for any number of classes and any number of decision trees, and will return class formulas as NNF circuits if the passed class formulas of decision trees are NNF circuits.
Since odd-even sorting networks have a depth of $O(\log^2(n))$ for $n$ inputs~\cite{10.1145/1468075.1468121}, Algorithm~\ref{alg:RF-classformula} generates class formulas of size $O(S + nk\log^2(n))$, where $n$ is the number of decision trees, $k$ is the number of classes, and $S$ is the total size of class formulas for the decision trees.

\subsection*{Tractable Class Formulas}

\begin{figure}[tb]
        \centering
        \scalebox{0.75}{
        \begin{tikzpicture}[
        roundnode/.style={circle ,draw=black, thick},
        squarednode/.style={rectangle, draw=black, thick},
        ]
        \node[squarednode]     (X)                              {\Large $X$};
        \node[squarednode]     (Y1)       [below=of X, xshift = -0.8cm, yshift = 0.25cm] {\Large $Y$};
        \node[squarednode]     (Y2)       [below=of X, xshift = 0.8cm, yshift = 0.25cm] {\Large $Y$};
        \node[squarednode]     (Y3)   [below=of Y2, xshift = -0.8cm, yshift = 0.25cm] {\Large $Y$};
        \node[roundnode]       (c1)   [below=of Y3, xshift = -1.2cm, yshift = 0.25cm] {\Large $c_1$};
        \node[roundnode]       (c2)   [below=of Y3, yshift = 0.25cm] {\Large $c_2$};
        \node[roundnode]       (c3)   [below=of Y3, xshift = 1.2cm, yshift = 0.25cm] {\Large $c_3$};
        
        \draw[-latex, thick] (X.240) -- node [anchor = center, xshift = -4mm, yshift = 1mm] {$x_{12}$} (Y1.north);
        \draw[-latex, thick] (X.300) -- node [anchor = center, xshift = 3mm, yshift = 1mm] {$x_{3}$} (Y2.north);
        \draw[-latex, thick] (Y1.270) -- node [anchor = center, xshift = -3mm, yshift = 1mm] {$y_{34}$} (c1.north);
        \draw[-latex, thick] (Y1.300) -- node [anchor = center, xshift = -2mm] {$y_{12}$} (Y3.north);
        \draw[-latex, thick] (Y2.240) -- node [anchor = center, xshift = 2.5mm] {$y_{34}$} (Y3.north);
        \draw[-latex, thick] (Y2.270) -- node [anchor = center, xshift = 3mm, yshift = 1mm] {$y_{12}$} (c3.north);
        \draw[-latex, thick] (Y3.240) -- node [anchor = center, xshift = -1mm, yshift = 2mm] {$y_1$} (c1.north);
        \draw[-latex, thick] (Y3.270) -- node [anchor = center, xshift = 1mm] {$y_{23}$} (c2.north);
        \draw[-latex, thick] (Y3.300) -- node [anchor = center, xshift = 1mm, yshift = 2mm] {$y_4$} (c3.north);
        \end{tikzpicture}
        }
        \scalebox{0.75}{
        \begin{tikzpicture}[
        roundnode/.style={circle ,draw=black, thick},
        squarednode/.style={rectangle, draw=black, thick},
        ]
        \node[squarednode]     (X)                              {\Large $X$};
        \node[squarednode]     (Y1)       [below=of X, xshift = -0.8cm, yshift = 0.25cm] {\Large $Y$};
        \node[squarednode]     (Y2)       [below=of X, xshift = 0.8cm, yshift = 0.25cm] {\Large $Y$};
        \node[squarednode]     (Y3)   [below=of Y1, xshift = 0.4cm, yshift = 0.25cm] {\Large $Y$};
        \node[squarednode]     (Y4)   [below=of Y2, xshift = -0.4cm, yshift = 0.25cm] {\Large $Y$};
        \node[roundnode]       (c1)   [below=of Y3, xshift = -0.8cm, yshift = 0.25cm] {\Large $c_1$};
        \node[roundnode]       (c2)   [below=of Y3, xshift = 0.4cm, yshift = 0.25cm] {\Large $c_2$};
        \node[roundnode]       (c3)   [below=of Y4, xshift = 0.8cm, yshift = 0.25cm] {\Large $c_3$};
        
        \draw[-latex, thick] (X.240) -- node [anchor = center, xshift = -4mm, yshift = 1mm] {$x_{12}$} (Y1.north);
        \draw[-latex, thick] (X.300) -- node [anchor = center, xshift = 3mm, yshift = 1mm] {$x_{3}$} (Y2.north);
        \draw[-latex, thick] (Y1.270) -- node [anchor = center, xshift = -3mm, yshift = 1mm] {$y_{34}$} (c1.north);
        \draw[-latex, thick] (Y1.300) -- node [anchor = center, xshift = 2.5mm] {$y_{12}$} (Y3.north);
        \draw[-latex, thick] (Y2.240) -- node [anchor = center, xshift = -2mm] {$y_{34}$} (Y4.north);
        \draw[-latex, thick] (Y2.270) -- node [anchor = center, xshift = 3mm, yshift = 1mm] {$y_{12}$} (c3.north);
        \draw[-latex, thick] (Y3.240) -- node [anchor = center, xshift = 1.5mm] {$y_1$} (c1.north);
        \draw[-latex, thick] (Y3.270) -- node [anchor = center, xshift = -1mm] {$y_{2}$} (c2.north);
        \draw[-latex, thick] (Y4.270) -- node [anchor = center, xshift = 1.5mm] {$y_{3}$} (c2.north);
        \draw[-latex, thick] (Y4.300) -- node [anchor = center, xshift = -1mm] {$y_4$} (c3.north);
        \end{tikzpicture}
        }
        \scalebox{0.75}{
        \begin{tikzpicture}[
        roundnode/.style={circle ,draw=black, thick},
        squarednode/.style={rectangle, draw=black, thick},
        ]
        \node[squarednode]     (X)                              {\Large $X$};
        \node[squarednode]     (Y1)       [below=of X, xshift = -0.8cm, yshift = 0.25cm] {\Large $Y$};
        \node[squarednode]     (Y2)       [below=of X, xshift = 0.8cm, yshift = 0.25cm] {\Large $Y$};
        \node[]     (Y3)   [below=of Y1, xshift = 0.4cm, yshift = 0.25cm] {};
        \node[]     (Y4)   [below=of Y2, xshift = -0.4cm, yshift = 0.25cm] {};
        \node[roundnode]       (c1)   [below=of Y3, xshift = -0.8cm, yshift = 0.25cm] {\Large $c_1$};
        \node[roundnode]       (c2)   [below=of Y3, xshift = 0.4cm, yshift = 0.25cm] {\Large $c_2$};
        \node[roundnode]       (c3)   [below=of Y4, xshift = 0.8cm, yshift = 0.25cm] {\Large $c_3$};
        
        \draw[-latex, thick] (X.240) -- node [anchor = center, xshift = -4mm, yshift = 1mm] {$x_{12}$} (Y1.north);
        \draw[-latex, thick] (X.300) -- node [anchor = center, xshift = 3mm, yshift = 1mm] {$x_{3}$} (Y2.north);
        \draw[-latex, thick] (Y1.270) -- node [anchor = center, xshift = -3mm] {$y_{134}$} (c1.north);
        \draw[-latex, thick] (Y1.300) -- node [anchor = center, xshift = -2mm] {$y_{2}$} (c2.north);
        \draw[-latex, thick] (Y2.240) -- node [anchor = center, xshift = 2mm] {$y_{3}$} (c2.north);
        \draw[-latex, thick] (Y2.270) -- node [anchor = center, xshift = 3mm] {$y_{124}$} (c3.north);
        \end{tikzpicture}
        }
        \caption{A decision graph (left), an equivalent weak test-once decision graph (middle), and an equivalent test-once decision graph (right). \label{fig:test-once}}
\end{figure}
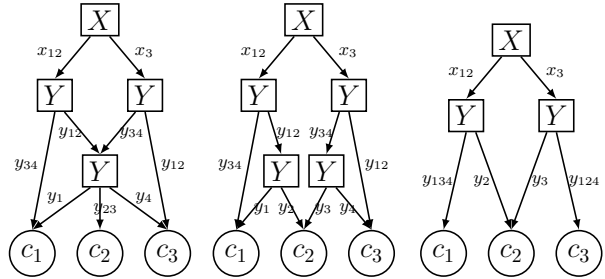

\begin{algorithm}[tb]
\begin{footnotesize}
\caption{Conjoin or disjoin two decision graphs
\label{alg:Apply}}
\begin{algorithmic}[1]
\Require Weak test-once decision graphs $D_1$ and $D_2$ whose leaves are true/false, path $p$, operation $op$
\Ensure A weak test-once decision graph $D_3 = D_1\ op\ D_2$ 
\Function{Apply}{$D_1$, $D_2$, $p$, $op$}
\If{$D_1$ or $D_2$ is True or False}
\State $D_3 \gets D_1\; op \; D_2$ \label{ln:Apply-proof}
\State Remove states inconsistent with path $p$ from $D_3$
\State \Return $D_3$
\EndIf
\State $p^\star \gets$ portion of path $p$ for variables in $D_1,D_2$
\If{$cache[D_1,D_2,p^\star,op] \neq nil$} 
\State \Return $cache[D_1,D_2,p^\star,op]$
\EndIf
\State $V_1,V_2 \gets$ variables tested at the roots of $D_1,D_2$ \label{line:branch}
\State $E \gets \{\}$ (empty set of edges)
\If{$V_1==V_2$}\label{line:test-same-variable}
\For{edge $({s_1},c_1)$ of the root of $D_1$ and edge $({s_2},c_2)$ of the root of $D_2$}
\State $s \gets s_1 \cap s_2 \cap p[V_1]$
\If{$s$ is not empty}
\State $p' \gets$ a copy of $p$ with $p'[V_1]=s$
\State $c \gets \Call{Apply}{c_1, c_2, p', op}$ \label{ln:Apply-proof-2}
\State Add edge $(s,c)$ to $E$
\EndIf
\EndFor
\Else
\For{edge $(s_1,c_1)$ of the root of $D_1$}
\State $s \gets s_1 \cap p[V_1]$
\If{$s$ is not empty}
\State $p' \gets$ a copy of $p$ with $p'[V_1] = s$
\State $c \gets \Call{Apply}{c_1, D_2, p', op}$ \label{ln:Apply-proof-3}
\State Add edge $(s,c)$ to $E$
\EndIf
\EndFor
\EndIf
\State $D_3 \gets$ decision node for variable $V_1$ with edges $E$
\State $cache[D_1,D_2,p^\star,op] \gets D_3$
\State \Return $D_3$ \label{line:return}
\EndFunction
\end{algorithmic}
\end{footnotesize}
\end{algorithm}

The method we just discussed allows us to construct class formulas as NNF circuits,
and quite efficiently as we show in Section~\ref{sec:experiments}. Our ultimate goal, however, is to use
these class formulas to compute complete and general reasons for decisions, which form a basis for
computing popular explanations like (general) sufficient/necessary reasons and contrastive explanations as discussed in~\cite{DBLP:conf/aaai/DarwicheJ22,DBLP:conf/jelia/JiD23}. If class formulas satisfy certain 
properties, then complete and general reasons can be computed in
time linear in the formula's size for any decision. One way to attain such properties 
is to represent the class formula as a decision graph that satisfies the weak-test once property, which we illustrate by an example next leaving the formal
definition to~\cite{DBLP:conf/aaai/DarwicheJ22}.
Consider the three decision graphs in Figure~\ref{fig:test-once}.
The test-once property says that a variable can be tested (i.e., appears) at most once on any path in the decision graph. The decision graph on the right satisfies the test-once property, and the two other decision graphs do not since the variable $Y$ appears twice on paths from the root to class $c_2$. The weak test-once property is less restrictive: it allows a variable to be tested multiple times on a path, but only if the test involves states of the variable that are feasible at the current testing node. Therefore, every test-once decision graph is also weak test-once. The decision graph on the left does not satisfy the weak test-once property: the state $y_3$ is not feasible for the second $Y$ variable on the path $X \xrightarrow{x_{12}} Y \xrightarrow{y_{12}} Y \xrightarrow{y_{23}} c_2$. We can make this decision graph weak test-once by creating an additional node for variable $Y$ and updating the states they test; see the middle decision graph of Figure~\ref{fig:test-once}.

We next present Algorithm~\ref{alg:Apply}, which can conjoin or disjoin two weak test-once decision graphs whose leaves are either true or false while preserving the weak test-once property.
The algorithm is based on a similar one for ordered binary decision diagrams (OBDDs)~\cite{10.1145/136035.136043} 
with a major exception: it takes an additional argument, path $p,$ which maps every variable $V$ to a set of its states, $p[V]$ (initially, $p[V]$ contains all states of variable $V$). The algorithm
assumes that a decision node for variable $V$ is represented as a set of edges $(s,c)$ where $s$ is a set of states for $V$ and $c$ is the decision node pointed to by this edge. The key insight here is that the path $p$ keeps track of feasible states for variables which are used to enforce the weak test-once property. 
We currently do not have a useful bound on the complexity of Algorithm~\ref{alg:Apply} but will evaluate it empirically in Section~\ref{sec:experiments}.\footnote{Our implementation of Algorithm~\ref{alg:Apply} includes optimizations and heuristics that are left out due to space limitations.} The proof of the correctness of Algorithm~\ref{alg:Apply} can be found in the Appendix.

We used Algorithm~\ref{alg:Apply} as follows. First, we construct each formula $C_l^i$ ($\neg C_l^i$) for decision tree $T_l$ and class $c_i$ in the form of a weak test-once decision graph by replacing the label $c_i$ in the tree by true (false) and all other labels by false (true). We then pass these as inputs to Algorithm~\ref{alg:RF-classformula} instead of passing NNF circuits. Second, instead of interpreting the max/min operations of a sorting network as disjoin/conjoin circuit constructors, we interpret them as the operations 
$\Call{Apply}{.,.,.,\vee}$ and $\Call{Apply}{.,.,.,\wedge}$ of Algorithm~\ref{alg:Apply} --- no longer constant-time operations. This way, the formulas $F^{ij}(T_1, \ldots, T_n)$ computed by
$\Call{Sortnet}{.,.}$
on Lines~\ref{ln:sortnet-out1} and~\ref{ln:sortnet-out2} of Algorithm~\ref{alg:RF-classformula}
will be decision graphs that are weak test-once.
If we further implement the conjoin operation on Line~\ref{ln:conjoin-F}
of Algorithm~\ref{alg:RF-classformula}
using $\Call{Apply}{.,.,.,\wedge}$, the 
computed class formulas by Algorithm~\ref{alg:RF-classformula} will be weak test-once decision graphs. We refer to the result as
Algorithm~\ref{alg:RF-classformula}A and use it 
when computing contrastive explanations. 
But we skip this step and refer to the result as Algorithm~\ref{alg:RF-classformula}B
for other tasks
since complete and general reasons can be computed efficiently over conjunctions of weak test-once decision graphs~\cite{DBLP:conf/aaai/DarwicheJ22,DBLP:conf/jelia/JiD23}.

\section{Computing the Robustness of Decisions}\label{sec:robustness}

\begin{algorithm}[tb]
\begin{footnotesize}
\caption{Compute robustness
\label{alg:Robustness-vars}}
\begin{algorithmic}[1]
\Require complete reason $\Delta$ in the form of an or-decomposable NNF circuit
\Ensure robustness $r$ and variables set $V$
\Function{Robustness}{$\Delta$}
\If{$\Delta$ is true}
\Return $\infty$, $\emptyset$
\ElsIf{$\Delta$ is false}
\Return $0$, $\{\emptyset\}$
\ElsIf{$\Delta$ is a literal}
\Return $1$, $\{\SetV(\Delta)\}$
\ElsIf{$\Delta$ is $\alpha \wedge \beta$}
\State $r_1, V_1 \gets \Call{Robustness}{\alpha}$
\State $r_2, V_2 \gets \Call{Robustness}{\beta}$
\If{$r_1 < r_2$}
\Return $r_1, V_1$
\ElsIf{$r_1 > r_2$}
\Return $r_2, V_2$
\EndIf
\Return $r_1, V_1 \cup V_2$
\ElsIf{$\Delta$ is $\alpha \vee \beta$}
\State $r_1, V_1 \gets \Call{Robustness}{\alpha}$
\State $r_2, V_2 \gets \Call{Robustness}{\beta}$
\State \Return $r_1 + r_2$, $\{v_1 \cup v_2 \mid v_1 \in V_1, v_2 \in V_2\}$
\EndIf
\EndFunction
\end{algorithmic}
\end{footnotesize}
\end{algorithm}

\begin{algorithm}[tb]
\begin{footnotesize}
\caption{Convert general reason into shortest clauses
\label{alg:sCNF}}
\begin{algorithmic}[1]
\Require general reason NNF $\Delta$, variables set $V$
\Ensure CNF $\Gamma$ with only the shortest clauses
\Function{sCNF}{$\Delta$, $V$}
\If{$\Delta$ is true or false}
\Return $\Delta$
\ElsIf{$\Delta$ is a literal}
\If{$\SetV(\Delta) \subseteq v_i$ for some $v_i \in V$}
\Return $\Delta$
\EndIf
\State \Return false
\ElsIf{$\Delta$ is $\alpha \wedge \beta$}
\State $\Gamma_1 \gets \Call{sCNF}{\alpha, V}$, $\Gamma_2 \gets \Call{sCNF}{\beta, V}$
\State $\Gamma = \Gamma_1 \wedge \Gamma_2$
\ElsIf{$\Delta$ is $\alpha \vee \beta$}
\State $\Gamma_1 \gets \Call{sCNF}{\alpha, V}$, $\Gamma_2 \gets \Call{sCNF}{\beta, V}$
\State $\Gamma =$ true
\For{every clause $c_1$ in $\Gamma_1$}
\For{every clause $c_2$ in $\Gamma_2$}
\State $c = c_1 \vee c_2$
\If{$\SetV(c) \subseteq v_i$ for some $v_i \in V$}
 $\Gamma = \Gamma \wedge c$ 
\EndIf
\EndFor
\EndFor
\EndIf
\State Remove subsumed clauses in $\Gamma$
\State \Return $\Gamma$
\EndFunction
\end{algorithmic}
\end{footnotesize}
\end{algorithm}

It can be quite beneficial to know whether the classification of an instance is robust to instance perturbations. The robustness of a decision on an instance $\instance$ in class $c$ was defined in~\cite{DBLP:conf/pgm/ShihCD18} as the minimum number of features needed to flip instance $\instance$ to a different class $c'$ --- it was also shown that robustness can be computed efficiently when the class formula is an Ordered Decision Diagram. We will next adopt this definition and show how to compute robustness based on the complete reason of a decision. We will further provide an algorithm for computing all ways in which a decision can be flipped using a smallest number of feature changes, based on the decision's general reason. Our proposed algorithms can therefore be applied to random forest classifiers, based on the results in Section~\ref{sec:sortnet} which facilitate the computation of complete and general reasons. 

Without loss of generality, we assume that every instance belongs to a single class in this section --- if an instance belongs to multiple classes, we can simply merge these classes. We start with some definitions. For instances $\instance$ and $\instance'$, we will use $\Dvars(\instance, \instance')$ to denote the set of variables that $\instance$ and $\instance'$ disagree on. 

\begin{definition}
Let $\instance$ be an instance, $\Delta$ be its class formula, $\sigma$ be a clause implied by $\instance,$
and let $\instance'$ be an instance such that $\vars(\sigma) = \Dvars(\instance, \instance').$
We say that $\instance'$ is a $\sigma$-violation of $\instance$ if $\instance' \not \models \sigma,$
and that $\instance'$ flips the decision on $\instance$ if $\instance' \not \models \Delta.$ 
\end{definition}

We also call a $\sigma$-violation $\instance'$ a ``violation identified by $\sigma$,'' and a ``flip identified by $\sigma$'' if it further flips the decision.

\begin{definition}
Let $\instance$ be an instance with class formula $\Delta.$ We say instance $\instance'$ is a shortest flip to the decision on $\instance$ if $\instance' \not \models \Delta$ and there does not exist another instance $\instance''$ such that $\instance'' \not \models \Delta$ and $\lvert \Dvars(\instance, \instance'') \rvert < \lvert \Dvars(\instance, \instance') \rvert$. 
\end{definition}

We next extend the definition of robustness in~\cite{DBLP:conf/pgm/ShihCD18} from binary to discrete.

\begin{definition}\label{def:robustness}
Let $\instance$ be an instance, and $\instance'$ be a shortest flip to the decision on $\instance$. The robustness of the decision on $\instance$ is $\lvert \Dvars(\instance, \instance') \rvert$.
\end{definition}
The robustness of a decision can be computed using necessary reasons. Recall that a necessary reason (NR) for a decision is a minimal property of the instance --- in the form of a clause --- that flips the decision if violated appropriately (by changing the instance). We next provide a proposition together with a lemma from~\cite{DBLP:conf/jelia/JiD23} to illustrate how robustness can be computed from necessary reasons.

\begin{algorithm}[tb]
\begin{footnotesize}
\caption{Compute shortest GNRs
\label{alg:sGNRs}}
\begin{algorithmic}[1]
\Require general reason NNF $\Delta$, variables set $V$
\Ensure a set $s$ contains all shortest GNRs 
\Function{sGNR}{$\Delta, V$}
\State $\Gamma = \Call{sCNF}{\Delta, V}$
\State $s \gets \emptyset$
\For{every $v_i \in V$}
\State $s_i \gets \{c | c \text{ is a clause in $\Gamma$ with } \SetV(c) = v_i\}$
\For{every variable $v$ in $v_i$}
\For{every pair of $c_1, c_2 \in s_i$}
\State $c \gets$ the resolvent of $c_1$ and $c_2$ on variable $v$
\If{$c$ is not subsumed by some clause in $s_i$}
\State Add $c$ to $s_i$
\State Remove clauses subsumed by $c$ in $s_i$
\EndIf
\EndFor
\EndFor
\State $s \gets s \cup s_i$
\EndFor
\State \Return $s$
\EndFunction
\end{algorithmic}
\end{footnotesize}
\end{algorithm}

\begin{lemma}\label{lemma:sNR}
Let $\sigma$ be an NR for the decision on instance $\instance$. Then there is a $\sigma$-violation that is a flip to the decision on $\instance$.
\end{lemma}

\begin{proposition}\label{prop:sNR}
Let $\instance$ be an instance, and $\instance'$ be a shortest flip to the decision on $\instance$. Then there must be an NR $\sigma$ such that $\instance'$ is a $\sigma$-violation of $\instance$.
\end{proposition}

Based on Lemma~\ref{lemma:sNR} and Proposition~\ref{prop:sNR}, an NR $\sigma$ identifies a shortest flip to the decision iff $\sigma$ is also shortest, i.e., there does not exist another NR $\sigma'$ such that $\lvert \SetV(\sigma') \rvert < \lvert \SetV(\sigma) \rvert$. Therefore, the robustness of a decision can be computed as the length (number of features) of the shortest NRs for a decision.

Algorithm~\ref{alg:Robustness-vars} computes the robustness of a decision from the complete reason based on this approach. The algorithm computes the minimal length of the prime implicates of the complete reason based on the minimal length of the prime implicates of its children. This approach is viable since the complete reason is monotone and or-decomposable;\footnote{The complete reason is monotone; see, e.g.,~\cite{DBLP:conf/lics/Darwiche23}.
Moreover, it is or-decomposable if computed from a weak test-once decision graph using the closed form in~\cite{DBLP:conf/aaai/DarwicheJ22}.} see Proposition~12 in~\cite{DBLP:conf/aaai/DarwicheJ22}. Algorithm~\ref{alg:Robustness-vars} is very efficient: with appropriate caching, it runs in time linear in the size of the complete reason. Algorithm~\ref{alg:Robustness-vars} also computes a set $V$ that contains the sets of variables of the shortest NRs. That is, $V = \{\SetV(\sigma) | \sigma \text{ is a shortest NR}\}$.  The set $V$ will be utilized later for computing all the shortest flips of a decision.

Necessary reasons are useful for determining robustness, as shown above, but they have a shortcoming when identifying changes to an instance that will flip a decision.
In particular, an NR $\sigma$ only guarantees that at least one $\sigma$-violation is a flip to the decision but not all such violations. Consider an instance $\instance = \{x_1, y_1\}$ and its class formula $\Delta = x_{12} \OR y_{12}$, where the NR for the decision is $x_1 \OR y_1$. Instance $\instance' = \{x_2, y_2\}$ is a violation of the NR, but it does not flip the decision. This was addressed in~\cite{DBLP:conf/jelia/JiD23} by introducing the notion of a general reason for a decision, which is a better instance abstraction than the complete reason when some features are non-binary. From the general reason, we can compute \textit{general necessary reasons} (GNRs), which have stronger properties compared to NRs, as illustrated by the next proposition together with a lemma from~\cite{DBLP:conf/jelia/JiD23}.

\begin{lemma}\label{lemma:sGNR}
Let $\sigma$ be a GNR for a decision on $\instance$. Then every $\sigma$-violation is a flip to the decision.
\end{lemma}

\begin{proposition}\label{prop:sGNR}
Let $\instance$ be an instance, and $\instance'$ be a shortest flip to the decision on $\instance$. Then there must be a GNR $\sigma$ such that $\instance'$ is a $\sigma$-violation of $\instance.$
\end{proposition}

Similar to NRs, a GNR $\sigma$ identifies shortest flips to the decision iff $\sigma$ is also shortest. The main difference is that Lemma~\ref{lemma:sGNR} ensures that every violation of a GNR is a flip to the decision. It then follows that the set of all shortest flips to a decision is exactly the set of all violations identified by the shortest GNRs. For the instance $\instance = \{x_1, y_1\}$ and its class formula $\Delta = x_{12} \OR y_{12}$, discussed above, the GNR for the decision is $x_{12} \OR y_{12}$, which tells us what the shortest flips are (i.e., all changes to the instance which violate GNR $x_{12} \OR y_{12}$). 

Algorithm~\ref{alg:sGNRs} uses Algorithm~\ref{alg:sCNF} as a helper and computes all shortest GNRs for a decision given the general reason $\Delta$ and the set $V$ computed by Algorithm~\ref{alg:Robustness-vars}. Algorithm~\ref{alg:sGNRs} uses the techniques in~\cite{DBLP:conf/jelia/JiD23} to compute all shortest GNRs, by first converting the general reason into a CNF and then closing it under resolution (while removing subsumed clauses). The helper Algorithm~\ref{alg:sCNF} converts the general reason $\Delta$ into a CNF with only the shortest clauses, which are clauses that contain only variables in the set $V$ obtained by Algorithm~\ref{alg:Robustness-vars}.
Then Algorithm~\ref{alg:sGNRs} divides the computed clauses into different sets based on their variables and closes each set under resolution. Moreover, Algorithm~\ref{alg:sGNRs} closes a set of clauses under resolution by assuming a variable order, resolving the clauses on a variable exhaustively, and then moving on to the next variable. This method is called \textit{variable depletion} and can be used to close a set of clauses under resolution in propositional logic more efficiently; see, e.g.,~\cite{Crama2011BooleanF}.\footnote{Our implementation of Algorithms~\ref{alg:sCNF} and~\ref{alg:sGNRs} includes some additional optimizations that are left out due to space limitations.} 

In summary, given a random forest classifier $R$ and an instance $\instance$ in class $c_i$, we can first compile the class formula $R^i$ as a tractable circuit using Algorithm~\ref{alg:RF-classformula}B. We can then compute the complete reason and the general reason for the decision in time linear to the size of $R^i$ using the closed forms introduced in~\cite{DBLP:conf/aaai/DarwicheJ22,DBLP:conf/jelia/JiD23}. Furthermore, all shortest flips to the decision on instance $\instance$ can be computed using Algorithm~\ref{alg:sGNRs} from the general reason and the set $V$ computed by Algorithm~\ref{alg:Robustness-vars}. Algorithm~\ref{alg:sCNF} simplifies the general reason by removing the non-shortest clauses, therefore, improving the efficiency of Algorithm~\ref{alg:sGNRs}. We will evaluate the performance of Algorithm~\ref{alg:sGNRs} empirically in Section~\ref{sec:experiments}. The proofs of the results in this section can be found in the Appendix. 

\section{Empirical Evaluation}\label{sec:experiments}

\begin{table*}[tb]
    \centering
    \begin{tabular}{l r r r r r r r r r r r}
    \toprule
    \multirow{2}{*}{\textbf{Dataset}}  & \multirow{2}{*}{\textbf{\#F}}  & \multirow{2}{*}{\textbf{\#C}} & \multirow{2}{*}{\textbf{\#I}}  & \multicolumn{4}{c}{\textbf{RF}} & \multicolumn{2}{c}{\textbf{NNF}} & \multicolumn{2}{c}{\textbf{ADD}} \\ \cmidrule(lr){5-8} \cmidrule(lr){9-10} \cmidrule(lr){11-12} 
    & & & & \textbf{\#T} &  \textbf{D} & \multicolumn{1}{c}{\textbf{\#$\text{N}_{\text{RF}}$}} & \multicolumn{1}{c}{\textbf{\%A}} & \multicolumn{1}{c}{\textbf{\#$\text{N}_{\text{NNF}}$}} &  \multicolumn{1}{c}{\textbf{$\text{T}_{\text{NNF}}$}} & \multicolumn{1}{c}{\textbf{\#$\text{N}_{\text{ADD}}$}} & \multicolumn{1}{c}{\textbf{$\text{T}_{\text{ADD}}$}}\\
    \midrule
    ann-thyroid  & 21 & 3 & 566 & 25 & 4 & 527 & 95.05 & 3529 & 0.20 & 210832 & 46.40 \\
    appendicitis & 7 & 2 & 16 & 50 & 4 & 630 &  68.75 & 1792 & 0.07  & 1118961 & 444.73 \\
    banknote & 4 & 2 & 206 & 100 & 4& 1996 & 96.60 & 5887 & 0.20 & 57066 & 64.39 \\
    ecoli & 7 & 5 & 49 & 100 & 4 &  1117 & 87.76 & 83106 & 7.11 & 1109022 & 3296.77\\
    glass2 & 9 & 2 & 24 & 25 & 4 & 465 & 83.33 & 965 & 0.04 & 420024 & 39.84 \\
    ionosphere & 34 & 2 & 53 & 15 & 4 & 53 & 84.91 & 504 & 0.01 & 264003 & 689.35 \\
    iris & 4 & 2 & 22 & 100 & 4 & 1084 & 95.45 & 25125 & 2.05 & 1802 & 2.58 \\
    magic & 10 & 2 & 2853 & 25 & 4 & 755 & 82.05 & 1370 & 0.06 & 628529 & 77.61\\
    mofn-3-7-10 & 10 & 2 & 199 & 100 & 4 & 2904 & 80.9 & 4100 & 0.21 & 15 & 0.16 \\
    new-thyroid & 5 & 3 & 32 & 100 & 4 & 1298 & 93.75 & 26188 & 2.35 & 94747 & 149.06 \\
    phoneme & 5 & 2 & 811 & 100 & 4 & 2870 & 80.39 & 6199 & 0.30 & 1410646 & 1788.00 \\
    ring & 20 & 2 & 1110 & 25 & 4 & 615 & 87.48 & 641 & 0.02 & 190706 & 150.60 \\
    segment & 19& 7& 346 & 15 & 4& 337 & 94.51 & 8734 & 0.55 & 9216490 & 1630.89 \\
    shuttle & 9& 7& 8700 & 50& 4& 1294 & 99.82 & 59899 & 5.13 & 477618 & 324.88 \\
    threeOf9 & 9 & 2 & 77 & 100 & 4 & 2830 & 98.70 & 4680 & 0.25 & 3 & 0.12 \\
    twonorm & 20 & 2 & 1110 & 15 & 4 & 465 & 90.63 & 835 & 0.01 & 9286188 & 2246.47 \\
    waveform-21 & 21 & 3 & 750 & 15 & 4 & 465 & 82.40 & 2202 & 0.08 & 33892041 & 593.29\\
    wine-recog & 13 & 3 & 27 & 25 & 4 & 403 & 100.00 & 3836 & 0.24 & 9124186 & 2544.07 \\
    xd6 & 9 & 2 & 146 & 100 & 4 & 2858 & 95.89 & 4610 & 0.20 & 46 & 0.22  \\
    \bottomrule
    \end{tabular}
    \caption{Comparing Algorithm~\ref{alg:RF-classformula} with results in~\cite{DBLP:conf/icaart/MurtoviSS25}. \textbf{\#F} is the number of features; \textbf{\#C} is the number of classes; \textbf{\#I} is the number of test instances. \textbf{\#T} is the number of trees, \textbf{D} is the max depth, \textbf{\#$\text{N}_{\text{RF}}$} is the number of nodes, and \textbf{\%A} is the classifier test accuracy of the random forest. \textbf{\#$\text{N}_{\text{NNF}}$} is the total number of nodes in the NNFs circuits of the random forest classifier for all classes, and \textbf{$\text{T}_{\text{NNF}}$} is the total time (in seconds) to generate them. \textbf{\#$\text{N}_{\text{ADD}}$} is the number of nodes of the ADD generated from the random forest classifier, and \textbf{$\text{T}_{\text{ADD}}$} is the time (in seconds) used to generate the ADD. The data in columns \textbf{\#$\text{N}_{\text{ADD}}$} and \textbf{$\text{T}_{\text{ADD}}$} is taken from~\cite{DBLP:conf/icaart/MurtoviSS25}. \label{table:experiments-compilation}}
\end{table*}

\begin{table*}[tb]
    \centering
    \begin{tabular}{@{}l r r r r r r r r r r r r r r r r r r r r@{}}
    \toprule
    \multirow{2}{*}{\textbf{Dataset}} & \multicolumn{3}{c}{\textbf{RF}} & \multicolumn{4}{c}{\textbf{Class Formula}}  \\ 
    \cmidrule(lr){2-4} \cmidrule(lr){5-8}
    & \textbf{\#T}  & \multicolumn{1}{c}{\textbf{\#$\text{N}_{\text{RF}}$}} & \multicolumn{1}{c}{\textbf{\%A}} & \multicolumn{1}{c}{\textbf{\#$\text{N}_{\text{3}}$}} &  \multicolumn{1}{c}{\textbf{$\text{T}_{\text{3}}$}} & \multicolumn{1}{c}{\textbf{\#$\text{N}_{\text{3B}}$}} & \multicolumn{1}{c}{\textbf{$\text{T}_{\text{3B}}$}} \\
    \midrule
    ann-thyroid  & 100 & 2092 & 96.82  & 9407.3 & 0.7 & 293176.0 & 2227.6 \\
    appendicitis & 100 & 1292 & 68.75 & 2298.0 &  0.1  & 310011.0 &  2056.3 \\
    banknote & 100 & 1996 & 96.60 & 3079.5 & 0.1 & 1782.5 & 14.4  \\
    ecoli & 50 & 1310 & 87.76 & 7248.0 & 0.6 &  37856.8 & 191.6  \\
    glass2 & 50 & 932 & 79.17 & 1190.5 & 0.1 & 402211.0 & 1539.4 \\
    ionosphere & 16 & 270 & 83.02 & 288.0 & 0.1 & 581794.5 & 454.7  \\
    iris & 100 & 1084 & 95.45  & 8851.0 & 0.7  & 118.7 & 3.9   \\
    magic & 32 & 968 & 82.23 & 1001.0 & 0.1 & 437863.5 & 502.0  \\
    mofn-3-7-10 & 100 & 2904 & 80.90  & 2061.5 & 0.1  & 22.0 & 0.5  \\
    new-thyroid & 100 & 1298 & 93.75  &  9423.3 & 0.8 & 10634.0 & 167.6 \\
    phoneme & 64 & 1816 & 80.27 & 1848.5 & 0.1  & 46576.5 & 287.9 \\
    ring & 25 & 615 & 87.48 & 321.5 & 0.1 & 206400.5 & 511.9 \\
    segment & 12 & 272 & 93.35 & 1276.7 & 0.1 & 270979.1 & 115.7  \\
    shuttle & 50 & 1294 & 99.82 & 9987.0 & 0.7 & 29636.7 & 125.9  \\
    threeOf9 & 100 & 2830 & 98.70 & 2354.5 & 0.1 & 39.0 & 1.1 \\
    twonorm & 10 & 310 & 89.19 & 299.0 & 0.1  & 195234.5 & 73.8  \\
    waveform-21 & 9 & 279 & 82.00  & 551.3 & 0.1  & 341053.0 & 156.1 \\
    wine-recog & 25 & 403 & 100.00 & 1576.0 & 0.1 & 2560010.7 & 3154.9  \\
    xd6 & 100 & 2858 & 95.89 & 2317.5 & 0.1 & 42.5 & 1.7  \\
    \bottomrule
    \end{tabular}
    \caption{Evaluating Algorithms~\ref{alg:RF-classformula}B for generating class formulas as conjunctions of decision graphs (DGs). \textbf{\#T} is the number of trees, \textbf{\#$\text{N}_{\text{RF}}$} is the number of nodes, and \textbf{\%A} is the classifier test accuracy of the random forest.  \textbf{\#$\text{N}_{\text{3}}$} is the number of nodes in each NNFs circuit of the random forest classifier for each class, and \textbf{$\text{T}_{\text{3}}$} is the average time (in seconds) to generate each of them using Algorithm~\ref{alg:RF-classformula}. \textbf{\#$\text{N}_{\text{3B}}$} is the average number of nodes of each conjunction of DGs for each class. \textbf{$\text{T}_{\text{3B}}$} is the average time (in seconds) to generate each class formula as a conjunction of DGs using Algorithms~\ref{alg:RF-classformula}B. \label{table:experiments-nt-t1}}
\end{table*}

\begin{table*}[t]
    \centering
    \begin{tabular}{@{}l r r r r r r r r r r r r r r r r r r r r@{}}
    \toprule
    \multirow{2}{*}{\textbf{Dataset}} & \multicolumn{3}{c}{\textbf{RF}} &  \multicolumn{2}{c}{\textbf{SR}} & \multicolumn{2}{c}{\textbf{NR}} & \multicolumn{3}{c}{\textbf{sGNR}}  \\ 
    \cmidrule(lr){2-4} \cmidrule(lr){5-6} \cmidrule(lr){7-8} \cmidrule(lr){9-11} 
    & \textbf{\#T}  & \multicolumn{1}{c}{\textbf{\#$\text{N}_{\text{RF}}$}} & \multicolumn{1}{c}{\textbf{\%A}}  &  \multicolumn{1}{c}{\textbf{$\text{T}_{\text{SR}}$}} & \multicolumn{1}{c}{\textbf{\#$\text{C}_{\text{SR}}$}}  & \multicolumn{1}{c}{\textbf{$\text{T}_{\text{NR}}$}} & \multicolumn{1}{c}{\textbf{\#$\text{C}_{\text{NR}}$}} &  \multicolumn{1}{c}{\textbf{r}} & \multicolumn{1}{c}{\textbf{$\text{T}_{\text{sGNR}}$}} & \multicolumn{1}{c}{\textbf{\#$\text{C}_{\text{sGNR}}$}} \\
    \midrule
    ann-thyroid  & 100 & 2092 & 96.82   & 0.0281 & 3.50 & 0.0090 & 3.58 & 1.25 & 0.1906 & 1.87  \\
    appendicitis & 100 & 1292 & 68.75   & 0.0100  & 6.44 & 0.0037 & 8.69 & 1.88 & 1.9259 & 19.75  \\
    banknote & 100 & 1996 & 96.60 & 0.0001  & 1.09 & 0.0001 & 1.85 &  1.00 & 0.0094 & 1.75 \\
    ecoli & 50 & 1310 & 87.76 &  0.0001 & 2.08 & 0.0001 & 3.76 & 1.00 & 0.0268 & 2.24\\
    glass2 & 50 & 932 & 79.17 &  0.0035 & 5.46 & 0.0015 & 7.83 & 1.21 & 0.4143 & 5.75 \\
    ionosphere & 16 & 270 & 83.02  & 77.2518  & 761.00  & 2.8630  & 777.89 & 1.73 &  0.1581 & 7.69 \\
    iris & 100 & 1084 & 95.45   & 0.0001 & 1.00 &  0.0001 & 1.59 & 1.00 & 0.0014 & 1.68\\
    magic & 32 & 968 & 82.23 & 0.0376 & 5.97 & 0.0105 & 8.55 & 1.19 & 0.5213 & 10.12\\
    mofn-3-7-10 & 100 & 2904 & 80.90   & 0.0001 & 3.98 & 0.0001 & 3.58 &  1.90 & 0.0001 & 2.55  \\
    new-thyroid & 100 & 1298 & 93.75 & 0.0001 & 1.59 &  0.0001 & 2.72 & 1.06 & 0.0098 & 2.50  \\
    phoneme & 64 & 1816 & 80.27   &  0.0004  & 2.88 & 0.0002  & 3.30 & 1.42 & 0.1229 & 7.17\\
    ring & 25 & 615 & 87.48 & 7.3803 & 132.04 & 1.1468  & 236.86 & 1.21 & 0.0165 & 3.19 \\
    segment & 12 & 272 & 93.35  & 0.3216 & 28.91 & 0.0716 & 38.91 & 1.24 & 0.4629 & 6.94 \\
    shuttle & 50 & 1294 & 99.82  & 0.0015 & 3.78 & 0.0007 & 5.37 & 1.15 & 0.0351 & 3.11\\
    threeOf9 & 100 & 2830 & 98.70  & 0.0001 & 2.75 & 0.0001 & 4.35 & 1.16 & 0.0001 & 2.68 \\
    twonorm & 10 & 310 & 89.19   & 3.0597 & 122.36 & 0.2581 & 125.24 &  1.66 & 0.1566 & 9.30 \\
    waveform-21 & 9 & 279 & 82.00   & 0.7104 & 44.62 & 0.1132 & 48.27 & 1.28 & 0.0321 & 4.53 \\
    wine-recog & 25 & 403 & 100.00   & 1.2963 & 40.33 &  0.3260 & 44.96 & 1.41 & 3.7344 & 7.59 \\
    xd6 & 100 & 2858 & 95.89  & 0.0003 & 4.62 & 0.0001 & 4.83 & 1.29 & 0.0001 & 3.02 \\
    \bottomrule
    \end{tabular}
    \caption{Evaluating the enumeration of SRs and NRs and Algorithm~\ref{alg:sGNRs} based on the complete and general reason computed from the tractable class formulas generated using Algorithms~\ref{alg:RF-classformula}B. \textbf{\#T} is the number of trees, \textbf{\#$\text{N}_{\text{RF}}$} is the number of nodes, and \textbf{\%A} is the classifier test accuracy of the random forest.  \textbf{$\text{T}_{\text{SR}}$} and \textbf{$\text{T}_{\text{NR}}$} are the average time to compute the SRs and NRs for each explained instance. \textbf{\#$\text{C}_{\text{SR}}$} and \textbf{\#$\text{C}_{\text{NR}}$} are the average numbers of SRs and NRs for each explained instance. \textbf{r} is the average robustness, \textbf{$\text{T}_{\text{sGNR}}$} is the average time to compute the shortest GNRs using Algorithm~\ref{alg:sGNRs}, and \textbf{\#$\text{C}_{\text{sGNR}}$} is the average number of shortest GNRs for each instance. \label{table:experiments-nt-t2}}
\end{table*}

\begin{table*}[t]
    \centering
    \begin{tabular}{@{}l r r r r r r r r r r r r r r r r r r r r@{}}
    \toprule
    \multirow{2}{*}{\textbf{Dataset}} & \multicolumn{1}{c}{\textbf{RF}} &  \multicolumn{4}{c}{\textbf{Class Formula}} &   \multicolumn{2}{c}{\textbf{$\text{NR}$}} & \multicolumn{2}{c}{\textbf{CE}} \\ 
    \cmidrule(lr){2-2} \cmidrule(lr){3-6}  \cmidrule(lr){7-8} \cmidrule(lr){9-10} 
    & \textbf{\#T} & \multicolumn{1}{c}{\textbf{$\text{T}_{\text{3B}}$}} & \multicolumn{1}{c}{\textbf{$\text{T}_{\text{3A}-\text{3B}}$}} & \multicolumn{1}{c}{\textbf{\#$\text{N}_{\text{3B}}$}} & \multicolumn{1}{c}{\textbf{\#$\text{N}_{\text{3A}}$}} & \multicolumn{1}{c}{\textbf{$\text{T}_{\text{NR}}$}} & \multicolumn{1}{c}{\textbf{\#$\text{C}_{\text{NR}}$}}  & \multicolumn{1}{c}{\textbf{$\text{T}_{\text{CE}}$}} & \multicolumn{1}{c}{\textbf{\#$\text{C}_{\text{CE}}$}}\\
    \midrule
    ann-thyroid & 100  & 2227.6 & 20.1 & 293176.0 & 220347.0 & 0.0090 & 3.58 & 0.0023 & 3.96\\
    ecoli & 50 & 191.6 & 1.4 & 37856.8 & 13483.6  & 0.0001 &  3.76 & 0.0001 & 2.81\\
    new-thyroid & 100 & 167.6 & 0.1 & 10634.0 & 4940.7 & 0.0001 & 2.72 & 0.0001 & 2.50  \\
    segment & 12 &  115.7 & 48.7 & 270979.1 & 184414.9 & 0.0716 & 38.91 &  0.0598  & 49.40 \\
    shuttle & 50 & 125.9 & 1.1 & 29636.9 & 12777.7 & 0.0007 & 5.37  & 0.0002 & 2.76 \\
    waveform-21 & 9 & 156.1  & 27.2 & 341053.0 & 190343.7 &  0.1132 & 48.27 &  0.0898 & 96.16\\
    wine-recog & 25 & 3154.9 & 125.2 & 2560010.7 & 1506961.7 & 0.3260  & 44.96 &  0.2514  & 38.87 \\
    \bottomrule
    \end{tabular}
    \caption{Evaluating Algorithm~\ref{alg:RF-classformula}A for generating class formulas as weak test-once decision graphs. \textbf{\#T} is the number of trees in the random forest classifier. \textbf{$\text{T}_{\text{3B}}$} is the average time used to generate each class formula using Algorithm~\ref{alg:RF-classformula}B in Table~\ref{table:experiments-nt-t1} and~\ref{table:experiments-nt-t2}, and \textbf{$\text{T}_{\text{3A}-\text{3B}}$} is the additional compilation time Algorithm~\ref{alg:RF-classformula}A took compared to Algorithm~\ref{alg:RF-classformula}B. \textbf{\#$\text{N}_{\text{3B}}$} and \textbf{\#$\text{N}_{\text{3A}}$} are the average number of nodes of the class formulas obtained from Algorithm~\ref{alg:RF-classformula}B and Algorithm~\ref{alg:RF-classformula}A respectively. \textbf{$\text{T}_{\text{NR}}$} and \textbf{$\text{T}_{\text{CE}}$} are the average time used to compute the NRs and CEs from their corresponding complete reasons.  \textbf{\#$\text{C}_{\text{NR}}$} and \textbf{\#$\text{C}_{\text{CE}}$} are the average numbers of NRs and CEs for each explained instance (and each target class).\label{table:experiments-CE}}
\end{table*}

We next evaluate our approaches for compiling class formulas into (tractable) circuits, and for computing the robustness of decisions, all SRs, NRs, and CEs from complete reasons, and all shortest flips (shortest GNRs) from general reasons. The experiments were performed on datasets from the UCI Machine Learning Repository \cite{Dua:2019} and the Penn Machine Learning Benchmarks \cite{Olson2017PMLB}. The random forest classifiers were learned by WEKA \cite{weka} using python-weka-wrapper3 available at pypi.org. We used 85\% of the dataset instances as training data and 15\% as testing data. We used a Python implementation of our algorithms on an Intel Xeon E5-2670 2.60GHz CPU with 256GB RAM. 

Table~\ref{table:experiments-compilation} summarizes the performance of Algorithm~\ref{alg:RF-classformula} in comparison with the experimental results reported in~\cite{DBLP:conf/icaart/MurtoviSS25}. The work in~\cite{DBLP:conf/icaart/MurtoviSS25} compiled a random forest classifier into an equivalent decision graph and then compared their approach with the ADD (Algebraic Decision Diagrams) compilation algorithm  in~\cite{DBLP:journals/sttt/GossenS23}. According to~\cite{DBLP:conf/icaart/MurtoviSS25}, their approach ORD has a geomean speedup of 13.2 in terms of compilation time and a geomean size increase of 3.03 across all datasets in Table~\ref{table:experiments-compilation}, compared with the ADD compilation algorithm in~\cite{DBLP:journals/sttt/GossenS23}. The class formulas of random forest classifiers can be obtained efficiently from the compiled ORD decision graph as NNF circuits~\cite{DBLP:conf/aaai/DarwicheJ22}, but the compiled ORD decision graph is not necessarily weak test-once and, thus, the complete and general reasons cannot be computed from these class formulas in linear time. As shown in Table~\ref{table:experiments-compilation}, we run Algorithm~\ref{alg:RF-classformula} with the same number of trees and depth reported in~\cite{DBLP:conf/icaart/MurtoviSS25} so that we can compare the sizes of the circuits generated. Among all tested datasets, the accuracy ranges from $80\%$ to $100\%$ with the only exception being the ``appendicitis'' dataset. The sizes of the NNF circuits generated are larger than the sizes of the ADDs for four simple datasets, among which three out of four are binary. Our NNFs are much smaller than the ADDs for all other datasets, and thus are also smaller than the sizes of the ORD decision graphs. Algorithm~\ref{alg:RF-classformula} is also efficient time-wise: it takes at most 7.11 seconds to compile the NNFs for any of the considered classifiers. The data in column \textbf{$\text{T}_{\text{ADD}}$} is taken from~\cite{DBLP:conf/icaart/MurtoviSS25} so we cannot run a direct time comparison, but the difference in compilation time is significant enough to conclude that Algorithm~\ref{alg:RF-classformula} is much more efficient than the ADD compilation approach~\cite{DBLP:journals/sttt/GossenS23} and the ORD compilation approach~\cite{DBLP:conf/icaart/MurtoviSS25}. We also tested the performance of Algorithm~\ref{alg:RF-classformula} on all the datasets in Table~\ref{table:experiments-compilation} with $1000$ trees and same depth. The random forest classifier for the ``segment'' dataset generated the largest NNF circuits of 3,837,629 nodes in total with the longest compilation time of 323.9 seconds. Algorithm~\ref{alg:RF-classformula} also has a similar but slightly better performance on the classifier for the ``shuttle'' dataset. This is due to segment and shuttle having the largest number of classes among the datasets we tested. These results show that Algorithm~\ref{alg:RF-classformula} scales very well relative to the number of trees in the random forest.

Table~\ref{table:experiments-nt-t1} and~\ref{table:experiments-nt-t2} summarize the performance of Algorithm~\ref{alg:RF-classformula}B for compiling class formulas as conjunctions of weak test-once decision graphs (DGs) on the same datasets with the same number of features, classes, test instances, and max depth. For each dataset, we explain up to $100$ test instances. We tested and discovered that the number of instances that belong to multiple classes was very low (less than 1\%). As such, for simplicity, we ordered the classes and treated the instances in multiple classes as if they belong to the highest ranked class that receives most votes. For each test instance, we generate the complete reason and the general reason using the closed-forms in~\cite{DBLP:conf/aaai/DarwicheJ22,DBLP:conf/jelia/JiD23}, then generate all the SRs and NRs from the complete reason using the algorithms in~\cite{DBLP:conf/aaai/DarwicheJ22} together with the robustness and shortest GNRs using Algorithms~\ref{alg:Robustness-vars} and~\ref{alg:sGNRs}. Unsurprisingly, Algorithm~\ref{alg:RF-classformula}B is less efficient and generates larger circuits compared to Algorithm~\ref{alg:RF-classformula} as it generates more tractable circuits; see~\cite{DBLP:journals/jair/DarwicheM02}. Nevertheless, Algorithm~\ref{alg:RF-classformula}B has reasonable performance: it compiles classifiers with at least 50 trees for 12/19 datasets we tested. The enumeration of SRs and NRs can also be done efficiently for most datasets, with two exceptions being ``ionosphere'' and ``ring'' --- ionosphere has the most features among all the datasets and consumed much more time to generate many more SRs and NRs. The enumeration of NRs is more efficient than the enumeration of SRs on these datasets. Moreover, a large compilation time of Algorithm~\ref{alg:RF-classformula}B does not necessarily imply a large SRs/NRs enumeration time, as seen by comparing datasets ann-thyroid, appendicitis, glass2 and magic with ionosphere and ring, but the size of the complete reason and the enumeration time appear correlated. We also found that the average robustness is very low for all the datasets we tested: the decisions can be flipped by modifying fewer than two features on average! We also computed the set of all shortest GNRs using general reasons and found that for 17/19 datasets, the number
of shortest GNRs per decision is less than $10$ so the set of all shortest decision flips can be captured using less than $10$ GSRs. Algorithm~\ref{alg:sGNRs} is also very efficient, as it has a similar overall performance to the computation of NRs.

Table~\ref{table:experiments-CE} summarizes our results for computing
contrastive explanations (CEs) on the multi-class datasets with the same settings in Table~\ref{table:experiments-nt-t1} and~\ref{table:experiments-nt-t2} --- for binary classifiers, CEs coincide with necessary reasons (NRs)
which we reported on in Table~\ref{table:experiments-nt-t1} and~\ref{table:experiments-nt-t2}.
Table~\ref{table:experiments-CE} shows the performance of Algorithm~\ref{alg:RF-classformula}A for compiling class formulas as weak test-once decision graphs, which is needed for computing CEs.
We compute the CEs for every test instance and possible target class using the following approach: for target class $c_i$, we negate the class formulas $R^i$  computed by Algorithm~\ref{alg:RF-classformula}A, then compute the complete reason from the negated class formula $\neg R^i$ (which is also a weak test-once decision graph) using the closed-form introduced in~\cite{DBLP:conf/aaai/DarwicheJ22}, and finally generate CEs as their prime implicates using also the algorithms in~\cite{DBLP:conf/aaai/DarwicheJ22}.\footnote{This approach is identical to the approach in \cite{DBLP:conf/aaai/DarwicheJ22,DBLP:conf/lics/Darwiche23} except for one point: for target class $c_i$, the approach in~\cite{DBLP:conf/aaai/DarwicheJ22,DBLP:conf/lics/Darwiche23} would merge all the remaining classes and compute a superclass $c_{\not = i}$ with class formula $R^{\not = i} = \bigvee_{j \neq i} R^j$. The basic idea is that kicking instance $\instance$ out of superclass $c_{\neq i}$ ensures it will end up in the target class $c_i$. Formulas $\neg R^i$ and $\bigvee_{j \neq i} R^j$ are equivalent under the assumption that all class formulas are mutually exclusive as in~\cite{DBLP:conf/aaai/DarwicheJ22,DBLP:conf/lics/Darwiche23}. However, this does not hold for random forest classifiers. If an instance is kicked out of  $\bigvee_{j \neq i} R^j$, it will land in a set where $c_i$ is the only new class of the modified instance. That is, it will miss changes that lead to instances which belong to class $c_i$ and some other tied class. As such, instead of merging all other class formulas, we choose to negate the class formula $R^i$, as this identifies all changes that lead to class $c_i$, including situations in which $c_i$ is not the unique class that the modified instance belongs to.}
As expected, Algorithm~\ref{alg:RF-classformula}A is relatively more expensive compared to Algorithm~\ref{alg:RF-classformula}B for classifiers with a fewer number of trees, as a random forest with fewer trees will generate sorting networks with fewer layers. Overall, the additional compilation cost of Algorithm~\ref{alg:RF-classformula}A is acceptable, and Algorithm~\ref{alg:RF-classformula}A generates smaller class formulas on all the datasets we tested. 
Finally, the enumeration of CEs was done efficiently for all the datasets we tested. 

\section{Conclusion}\label{sec:conclusion}

We proposed two algorithms for compiling the class formulas of a random forest classifier. The first algorithm generates class formulas as NNF circuits, and the experimental results show that this algorithm is significantly more efficient compared to existing approaches for obtaining class formulas as circuits. The second algorithm converts the generated NNF circuits into tractable decision graphs, which allow us to compute the complete and general reason for a decision in time linear in the graph sizes. Our approach forms a basis for explaining decisions made by random forest classifiers by allowing us to enumerate various explanations 
for these decisions (sufficient, necessary, and contrastive) and to compute decision robustness together with all shortest ways to flip the decision. 
Our empirical evaluation showed that all these computations can be done quite efficiently on the datasets we tested.

\nocite{lipton_1990,DBLP:journals/ai/Miller19,DBLP:conf/aiia/IgnatievNA020,DBLP:conf/aaai/DarwicheJ22,DBLP:conf/lics/Darwiche23,DBLP:conf/icaart/MurtoviSS25,DBLP:journals/sttt/GossenS23,DBLP:journals/jair/DarwicheM02,Dua:2019,Olson2017PMLB,weka}

\vspace{3mm}
\noindent{\bf Ack.} This work has been partially supported by ARL grant \#W911NF2510095.

\bibliographystyle{splncs04}
\bibliography{references}

\appendix 

\section{Proofs}\label{sec:supplementary-proofs}

We first prove the correctness of Algorithm~\ref{alg:Apply}.

\begin{proof}[Proof of the correctness of Algorithm~\ref{alg:Apply}]
Let $D_1$ and $D_2$ be two weak test-once decision graphs whose leaves are either true or false. Let $D_3$ be the output of $\Call{Apply}{D_1, D_2, p, op}$, where $p$ contains all states of every variable in $D_1$ and $D_2$, and op is either $\wedge$ or $\vee$. Our goal is to prove that $D_3$ is a weak test-once decision graph whose leaves are either true or false, and it is equivalent to $D_1\; op \; D_2$.

We first prove that $D_3$ is a weak test-once decision graph whose leaves are either true or false. It should be immediate that the leaves of $D_3$ are either true or false, since each leaf of $D_3$ is obtained by conjoining or disjoining some leaf of $D_1$ or $D_2$ with true or false (line~\ref{ln:Apply-proof} of Algorithm~\ref{alg:Apply}). Moreover, it should also be immediate that $D_3$ satisfy the weak test-once property, since during each recursive call of $\Call{Apply}{D_1, D_2, p, op}$, $p$ keeps tracks of all the legal states that can appear without violating the weak test-once property.

We next prove that $D_3$ is equivalent to $D_1\; op \; D_2$. Consider an arbitrary instance $\instance$. $D_1$ classifies $\instance$ by going through a path $p_1$ in $D_1$ that leads to leaf $l_1$, which is either true or false. $D_2$ and $D_3$ classifies $\instance$ similarly by going through paths $p_2$ and $p_3$ that leads to $l_2$ and $l_3$. Our goal is to prove that $l_3 = l_1\; op \;l_2$, and we do so by the following recursive procedure.

\begin{enumerate}
    \item If $D_1$ is $l_1$ (true/false) or $D_2$ is $l_2$, then by line~\ref{ln:Apply-proof} of Algorithm~\ref{alg:Apply}, $D_3$ is exactly $D_1\; op \; D_2$ with states inconsistent with $p$ removed. Since $p$ is always consistent with $\instance$, we have $l_3 = l_1\; op \;l_2$.
    \item Otherwise, since $\instance$ must be consistent with path $p$, $p_1$, and $p_2$, by either line~\ref{ln:Apply-proof-2} or line~\ref{ln:Apply-proof-3}, $D_3$ will contain an edge consistent with $\instance$ pointing to a child $c = \Call{Apply}{c_1, c_2, p', op}$ or $\Call{Apply}{c_1, D_2, p', op}$, where $c_1$ is the next node in path $p_1$, $c_2$ is the next node in path $p_2$, and $c$ is the next node in path $p_3$. Moreover, $p'$ must also be consistent with $\instance$. Thus, we can let $D_1 = c_1$, $D_2 = c_2$ or $D_2$ (depending on whether $c$ was generated by line~\ref{ln:Apply-proof-2} or line~\ref{ln:Apply-proof-3}), $D_3 = c$, and $p = p'$. Although we modified $D_1, D_2$, and $D_3$, we simply moved forward by one step on paths $p_1, p_2, p_3$, so $l_1, l_2$, and $l_3$ do not change. If we can show in this new setting, $l_3 = l_1\; op \;l_2$, then we finish the proof. We next go back to the previous step with the modified setting.
\end{enumerate}

Since the paths $p_1$ and $p_2$ must be finite, the procedure above will always terminate with $l_3 = l_1\; op \;l_2$ guaranteed by line~\ref{ln:Apply-proof} of Algorithm~\ref{alg:Apply}. 
\end{proof}

We next prove Proposition~\ref{prop:sGNR} then Proposition~\ref{prop:sNR}.

\begin{proof}[Proof of Proposition~\ref{prop:sGNR}]
Let $\instance$ be an instance and $\instance'$ be a shortest flip to the decision on $\instance$.
Let $\tau$ be a term consisted of literals in $\instance'$ such that $\SetV(\tau) = \Dvars(\instance, \instance')$. Then $\neg \tau$ is a clause that satisfies the requirements in Definition~8 in~\cite{DBLP:conf/jelia/JiD23}. Therefore, there must be a GNR $\sigma$ satisfies $\sigma \models \neg \tau$. Since $\instance'$ is a shortest flip, we have $\SetV(\neg \tau) = \SetV(\sigma)$. Thus, $\instance'$ is a $\sigma$-violation.
\end{proof}

\begin{proof}[Proof of Proposition~\ref{prop:sNR}]
Let $\instance$ be an instance and $\instance'$ be a shortest flip to the decision on $\instance$.
By Proposition~\ref{prop:sGNR}, there must be a GNR $\sigma'$ such that $\instance'$ is a $\sigma'$-violation. Let $\sigma$ be a clause obtained by disjoining all states that appear both in $\instance$ and $\sigma'$. Then by Proposition~11 in~\cite{DBLP:conf/jelia/JiD23}, $\sigma$ must be an NR. Furthermore, we must have $\SetV(\sigma) =\SetV(\sigma')$, since $\sigma'$ is a shortest GNR and $\sigma$ cannot identify an even shortest flip. Hence, $\instance'$ must be a $\sigma$-violation.
\end{proof}

We next prove the correctness of Algorithm~\ref{alg:sGNRs}. We first show that the variables set $V$ is universal. That is, the set of variables in NRs is the same as the set of variables in GNRs.

\begin{lemma}\label{lemma:universal-v}
$\{\SetV(\sigma) | \sigma \text{ is a shortest NR}\} = \{\SetV(\sigma) | \sigma \text{ is a shortest GNR}\}$.
\end{lemma}

\begin{proof}[Proof of Lemma~\ref{lemma:universal-v}]
Let $S_{NR}$ denote $\{\SetV(\sigma) | \sigma \text{ is a shortest NR}\}$ and $S_{GNR}$ denote $\{\SetV(\sigma) | \sigma \text{ is a shortest GNR}\}$.
Let $\sigma$ be an arbitrary shortest NR. Then there is a $\sigma$-violation $\instance'$ that is a shortest flip to the decision. This shortest flip must also be identified by a shortest GNR, so we have $S_{NR} \subseteq S_{GNR}$. The reverse direction works similarly, which gives us $S_{NR} \supseteq S_{GNR}$. This concludes the proof.
\end{proof}

We next show why the set of all shortest GNRs can be computed by converting the general reason into a CNF with only the shortest clauses and then closing it under resolution. We first go over the definition of resolution in discrete logic and the fixation property introduced in~\cite{DBLP:conf/jelia/JiD23}.

\begin{definition}\label{def:resolution}
Let $\alpha = l_1 \OR \sigma_1$, 
$\beta = l_2 \OR \sigma_2$ be two clauses
where $l_1$ and $l_2$ are $X$-literals (literals of variable $X$).
If \(\sigma = (l_1 \AND l_2) \OR \sigma_1 \OR \sigma_2 \neq \) true, then
the \(X\)-resolvent of clauses $\alpha$ and $\beta$ is defined as the clause equivalent to \(\sigma\).
\end{definition}

\begin{definition}\label{def:fixation}
An NNF is locally fixated on instance $\instance$  
iff every literal in the NNF is consistent with~$\instance$.
\end{definition}

\begin{corollary}\label{cor:resolution-subsumption}
Let $\alpha$, $\alpha'$, and $\beta$ be three clauses where $\alpha$ subsumes $\alpha'$. Then the resolvent of $\alpha$ and $\beta$ on variable $X$ subsumes the resolvent of $\alpha'$ and $\beta$ on variable $X$.
\end{corollary}

The corollary above follows from Definition~\ref{def:resolution}. It essentially says that if a clause is subsumed by another clause, then the subsumed clause becomes irrelevant in the sense that the resolvents it generates must also be subsumed by other resolvents. Hence, if we are closing a set of clauses under resolution, the order of subsumption checks doesn't matter. We can do incremental subsumption check in the process, or we can do one final subsumption check after the resolutions are over, and we will always get the same results. We will use this result implicitly in the following proofs. We then introduce the following lemma from~\cite{DBLP:conf/jelia/JiD23}, which is crucial for Algorithm~\ref{alg:sGNRs}.

\begin{lemma}\label{lemma:inc vd clauses}
Let $\alpha=l_1 \OR \sigma_1$, 
$\beta=l_2 \OR \sigma_2$ be two clauses which are locally fixated on some instance $\instance$. 
If $l_1$ and $l_2$ are \(X\)-literals, and if
\(\sigma\) is the \(X\)-resolvent
of clauses \(\alpha\) and \(\beta\), then
$\SetV(\sigma) = \SetV(\alpha) \cup \SetV(\beta)$.
\end{lemma}

This lemma allows us to prove that only the shortest clauses in the general reason CNF is needed to compute the set of all shortest GNRs.

\begin{proposition}\label{prop:alg-sGNR-correctness}
Let NNF $\Delta$ be a general reason for the decision on $\instance$ and set $V =  \{\SetV(\sigma) | \sigma \text{ is a shortest GNR}\}$. Let $S$ denote the set of clauses obtained by converting $\Delta$ into a CNF. For each $v_i \in V$, let $S_i$ denote the set $\{c | c \in S, \SetV(c) = v_i\}$, and let $S_i^*$ denote the set obtained by closing $S_i$ under resolution while removing subsumed clauses. Then the set of all shortest GNRs for the decision  is equal to $\bigcup_{v_i \in V} S_i^*$.
\end{proposition}

\begin{proof}[Proof of Proposition~\ref{prop:alg-sGNR-correctness}]
This proof is based on the following key observation: by Lemma~\ref{lemma:inc vd clauses}, the length of resolvents generated from clauses locally fixated on $\instance$ cannot decrease. Thus, if we close a set of clauses under resolution while removing subsumed clauses, the shortest clauses in the end must can be generated using only the shortest clauses in the beginning. We next begin our proof.

Let $S^*$ be the set of clauses obtained by closing $S$ under resolution while removing subsumed clauses. Note that $\Delta$ is locally fixated by Corollary~1 and Proposition~13 in~\cite{DBLP:conf/jelia/JiD23}, so every clause in $S$ must be locally fixated on $\instance$ because  converting an NNF into an CNF preserves the local fixation property (we are only considering the standard CNF conversion algorithm here). Furthermore, by Definition~\ref{def:resolution}, if two clauses are both locally fixated on $\instance$, then the resolvent must also be locally fixated on $\instance$. Thus, by induction, every resolvent generated during the resolution process is locally fixated on $\instance$. 

By Propositions~18 and~19 in~\cite{DBLP:conf/jelia/JiD23}, the shortest clauses in $S^*$ are exactly the shortest GNRs of the decision. Let $l$ be the minimal length of clauses in $S^*$. That is, $l = \lvert v_i \rvert$ for every $v_i \in V$. Note that the minimal length of clauses in $S$ must also be $l$. Now, consider the current set of clauses, $S'$, at some instant of the resolution process of closing $S$ under resolution. By Lemma~\ref{lemma:inc vd clauses}, any-resolvent $c$ generated from two clauses $c_1, c_2$ in $S'$ has $\lvert \SetV(c) \rvert = l$ if and only if $\SetV(c_1) = \SetV(c_2)$ and $\lvert \SetV(c_1) \rvert = l$.

Therefore, for every shortest clause $c$ in $S^*$ with $\SetV(c) = v_i$, $c$ must either be in $S_i$, or $c$ is a resolvent generated using clauses $c_1'$ and $c_2$ satisfying $\SetV(c_1') = \SetV(c_2') = v_i$. This statement is also true for $c_1'$ (and $c_2'$), since $c_1'$ must either be in $S_i$, or $c_1'$ is a resolvent generated using clauses $c_1''$ and $c_2''$ satisfying $\SetV(c_1'') = \SetV(c_2'') = v_i$. Since every shortest clause $c$ in $S^*$ can be obtained from $S$ by a finite number of resolutions, by induction, $c$ can be obtained from $S_i$ by closing it under resolution. Finally, since every clause in $S_i^*$ can also be generated by closing $S$ under resolution, $c$ not being subsumed by any other clause in $S^*$ indicates that $c$ must not be subsumed by any distinct clause in $S_i^*$. Hence, we have $c \in S_i^*$. This concludes one direction of the proof as it shows that every shortest GNR is in some $S_i^*$. We next prove the other direction.

Let $c$ be a clause in $S_i^*$ for some $v_i \in V$. We want to show that $c$ must be a shortest GNR. Since $c$ already satisfies $\SetV(c) = v_i$ and $c$ can be obtained by closing $S_i$, we only need to prove that $c$ cannot be subsumed by any distinct clause in $S^*$. This step is immediate since a clause $c_1$ can only be subsumed by a clause $c_2$ if $\SetV(c_1) \supseteq \SetV(c_2)$. Hence, if $c$ is not subsumed by any other clause in $S_i^*$, it must not be subsumed by any distinct clause in $S^*$.
\end{proof}

Proposition~\ref{prop:alg-sGNR-correctness} indicates that the main idea of Algorithm~\ref{alg:sGNRs} is correct: only the shortest clauses of a general reason CNF are needed to generate the shortest GNRs. We next show that variable depletion can be extended to discrete logic.

\begin{proposition}\label{prop:vd-discrete}
Let $S$ be a set of clauses, $O$ be a variable order of all variables appearing in $S$, and $T^*$ be a set of clauses obtained via the following procedures that we call variable depletion:
\begin{itemize}
    \item For each variable $o_i$ in $O$:
    \begin{itemize}
        \item For every pair of clauses $c_1, c_2$ in $S$:
        \begin{itemize}
            \item Resolve $c_1$ and $c_2$ on $o_i$, if possible
            \item Add the resolvent to $S$
            \item Remove subsumed clauses in $S$
        \end{itemize}
    \end{itemize}
\end{itemize}
Then $T^*$ is equal to the set of clauses obtained by closing $S$ under resolution while removing subsumed clauses.
\end{proposition}

\begin{proof}[Proof of Proposition~\ref{prop:vd-discrete}]
Let $S^*$ be the set of clauses obtained by closing $S$ under resolution while removing subsumed clauses. By corollary~\ref{cor:resolution-subsumption}, we can assume that subsumption checks happen after all resolutions are all finished for both approaches. Let $S'$ be the set of clauses obtained by closing $S$ under resolution without removing subsumed clauses, and let $T'$ be the set of clauses obtained by variable depletion without removing subsumed clauses. Our goal is to show that $S' = T'$.

It is immediate that $T' \subseteq S'$, as every resolvent generated by variable depletion can be generated by closing $S$ under resolution exhaustively. We next show that $S' \subseteq T'$.

Consider three clauses $\alpha = l_1 \OR r_1 \vee \sigma_1$, $\beta = l_2 \OR r_2 \vee \sigma_2$, and $\gamma = l_3 \OR r_3 \vee \sigma_3$, where $l_1, l_2, l_3$ are $X$-literals and $r_1, r_2, r_3$ are $Y$-literals (false if the clause does not mention the variable). Then the $X$-resolvent of $\alpha$ and $\beta$ is $(l_1 \wedge l_2) \OR r_1 \OR r_2 \vee \sigma_1 \vee \sigma_2$. We can further obtain the $Y$-resolvent of this resolvent with $\gamma$ as $[(l_1 \wedge l_2) \vee l_3] \OR [(r_1 \OR r_2) \wedge r_3] \vee \sigma_1 \vee \sigma_2 \vee \sigma_3$. Observe that we can also get two $Y$-resolvents by resolving $\alpha$ with $\gamma$ and $\beta$ with $\gamma$: this gives us $l_1 \vee l_3 \OR (r_1 \wedge r_3) \vee \sigma_1 \vee \sigma_3$ and $l_2 \vee l_3 \OR (r_2 \wedge r_3) \vee \sigma_2 \vee \sigma_3$. We can then further resolve these two resolvents on $X$ to get the following clause: $[(l_1 \vee l_3) \wedge (l_2 \vee l_3)] \OR [(r_1 \wedge r_3) \vee (r_1 \wedge r_3)] \vee \sigma_1 \vee \sigma_2\vee \sigma_3 = [(l_1 \wedge l_2) \vee l_3] \OR [(r_1 \OR r_2) \wedge r_3] \vee \sigma_1 \vee \sigma_2 \vee \sigma_3$. Therefore, the resolvent of resolutions on three clauses sequentially on $X$ first then $Y$ is equal to the resolvent of some other resolutions on $Y$ first then $X$. This shows that the order of resolution can be ``swapped'' for any two variables, so we can also first resolve exhaustively on a variable and then move on to the next variable. Hence, for any resolvent in $S'$, the clause can also be obtained using variable depletion w.r.t. a given variable order $O$. This shows that $S' \subseteq T'$ and concludes the proof.
\end{proof}

We finally concludes the correctness of Algorithm~\ref{alg:sGNRs}

\begin{proposition}\label{prop:alg-sGNR-final-correctness}
Let NNF $\Delta_c$ be the complete reason for the decision an instance, NNF $\Delta_g$ be the general reason for the decision, and $V$ be the set of variables computed by $\Call{Robustness}{\Delta_c}$. Then $\Call{sGNR}{\Delta, V}$ outputs the set of all shortest GNRs for the decision.
\end{proposition}

\begin{proof}[Proof of Proposition~\ref{prop:alg-sGNR-final-correctness}]
Follows from Lemma~\ref{lemma:universal-v} and Propositions~\ref{prop:alg-sGNR-correctness}, \ref{prop:vd-discrete}.
\end{proof} 

\end{document}